\documentclass{article}

\usepackage{microtype}
\usepackage{graphicx}
\usepackage{subcaption}
\usepackage{booktabs} % for professional tables

\usepackage{titletoc}
\usepackage{xcolor}
\usepackage{mathrsfs,mdwlist,multirow,colortbl,bm}
\usepackage{diagbox}
\usepackage{adjustbox}
\usepackage{pifont}

\usepackage{tcolorbox}

\usepackage{hyperref}

\usepackage[preprint]{icml2026}

\usepackage{amsmath}
\usepackage{amssymb}
\usepackage{mathtools}
\usepackage{amsthm}

\usepackage[capitalize,noabbrev]{cleveref}

\theoremstyle{plain}
\newtheorem{theorem}{Theorem}[section]
\newtheorem{proposition}[theorem]{Proposition}
\newtheorem{lemma}[theorem]{Lemma}

\theoremstyle{definition}

\newtheorem{assumption}[theorem]{Assumption}
\theoremstyle{remark}
\newtheorem{remark}[theorem]{Remark}

\usepackage[textsize=tiny]{todonotes}

\icmltitlerunning{Omni-Masked Gradient Descent: Memory-Efficient Optimization via Mask Traversal  with Improved Convergence}

\begin{document}

\twocolumn[
% \icmltitle{Two-Way Without-Replacement Masked Training\\
% for Stochastic Gradient Descent}
% \icmltitle{Omni-Mask: Memory-Efficient Optimization via Mask Traversal \\ with Improved Convergence}
% \icmltitle{OMG: Omni-Masked Gradients for Memory-Efficient Optimization with Improved Iteration Complexity}
\icmltitle{Omni-Masked Gradient Descent: Memory-Efficient Optimization \\ via Mask Traversal  with Improved Convergence}



  \icmlsetsymbol{equal}{*}

  \begin{icmlauthorlist}
    \icmlauthor{Hui Yang}{equal,gsm}
    \icmlauthor{Tao Ren}{equal,gsm}
    \icmlauthor{Jinyang Jiang}{equal,gsm}
    \icmlauthor{Wan Tian}{wangxuan}
    \icmlauthor{Yijie Peng}{gsm}
  \end{icmlauthorlist}

  \icmlaffiliation{gsm}{Guanghua School of Management, Peking University}
  \icmlaffiliation{wangxuan}{Wangxuan Institute of Computer Technology, Peking University}

  \icmlcorrespondingauthor{Yijie Peng}{pengyijie@pku.edu.cn}

  \icmlkeywords{Machine Learning, ICML}

  \vskip 0.3in ]



\printAffiliationsAndNotice{}  % no special notice (required even if empty)

\begin{abstract}
  % This document provides a basic paper template and submission guidelines.
  % Abstracts must be a single paragraph, ideally between 4--6 sentences long.
  % Gross violations will trigger corrections at the camera-ready phase.
  Memory-efficient optimization methods have recently gained increasing attention for scaling full-parameter training of large language models under the GPU-memory bottleneck. Existing approaches either lack clear convergence guarantees, or only achieve the standard ${\mathcal{O}}(\epsilon^{-4})$ iteration complexity in the nonconvex settings. We propose \textbf{O}mni-\textbf{M}asked \textbf{G}radient \textbf{D}escent (\textbf{OMGD}), an optimization method based on mask traversal for memory efficient training, and provide a nonconvex convergence analysis that establishes a strictly improved iteration complexity of $\tilde{\mathcal{O}}(\epsilon^{-3})$ for finding an $\epsilon$-approximate stationary point. Empirically, OMGD is a lightweight, plug-and-play approach that integrates seamlessly into most mainstream optimizers, yielding consistent improvements over competitive baselines in both fine-tuning and pre-training tasks. Our code is available at \verb|https://github.com/yh6-coder/OMGD|.
\end{abstract}

\section{Introduction}

Large language models (LLMs) have become the foundation of many recent advances in natural language processing, and training them at scale almost always relies on stochastic-gradient-based optimization methods such as SGD and Adam~\cite{bottou2010large, kingma2017adammethodstochasticoptimization}. Mini-batch stochastic gradients compute updates from only a small subset of training examples, reducing the per-step cost and enabling training on large datasets, while the injected randomness can improve exploration and help escape poor local minima and saddle points~\cite{pmlr-v80-kleinberg18a,chaudhari2019entropy}. However, for dense transformer-based LLMs, full-parameter training is severely GPU-memory bound: model parameters, activations, gradients, and optimizer states must all reside in device memory. For example, full-parameter training of a 7B-parameter model typically requires at least 60\,GB of GPU memory with Adam on a single device~\cite{pan2024lisalayerwiseimportancesampling}. To relieve this pressure, existing methods mainly follow two directions: \textbf{parameter-efficient fine-tuning (PEFT)} that updates only a small subset of parameters, such as LoRA~\cite{Hu2021LoRA}, QLoRA~\cite{Dettmers2023QLoRA}, SIFT~\cite{song2024sparsefinetuningpretrainedlarge}, LISA~\cite{pan2024lisalayerwiseimportancesampling}, and GMT~\cite{li2025enhancinglargelanguagemodel}; and \textbf{gradient/optimizer state compression} that shrinks gradient and optimizer states while still performing full-parameter updates, such as GaLore~\cite{Zhao2024GaLore} and GoLore~\cite{He2024GoLore}.

% Despite these advances, existing memory-efficient methods exhibit notable limitations from the perspective of optimization theory. (\emph{i}) Many masking/subspace update schemes are largely heuristic and come without a clear convergence characterization, and dominated-subspace updates can even be \emph{nonconvergent} under standard assumptions due to persistent bias from repeatedly optimizing in a dominated low-dimensional space, such as GaLore~\cite{Zhao2024GaLore}, SIFT~\cite{song2024sparsefinetuningpretrainedlarge}, and GMT~\cite{li2025enhancinglargelanguagemodel}. (\emph{ii}) When convergence theory is available, it often relies on a strong assumption of convex objectives, such as LISA~\cite{pan2024lisalayerwiseimportancesampling}. (\emph{iii}) Even when nonconvex convergence is established, the resulting iteration complexity can remain at the standard ${\mathcal{O}}(\epsilon^{-4})$ level for finding an $\epsilon$-approximate stationary point, such as GoLore~\cite{He2024GoLore}. These gaps suggest that memory savings alone do not explain (or guarantee) favorable optimization dynamics for large-scale nonconvex training.

Despite these advances, existing memory-efficient methods exhibit notable limitations from the perspective of optimization theory. ({i}) Many masking or subspace update methods are largely heuristic and come without a clear convergence characterization, and dominated-subspace updates can even be \emph{nonconvergent} under standard assumptions due to persistent bias from repeatedly optimizing in a dominated low-dimensional space, such as GaLore, SIFT, and GMT. ({ii}) When convergence theory is available, it often relies on a strong assumption of convex objectives, such as LISA. ({iii}) Even when nonconvex convergence is established, the resulting iteration complexity can remain at the standard ${\mathcal{O}}(\epsilon^{-4})$ level for finding an $\epsilon$-approximate stationary point, such as GoLore. These gaps suggest that memory savings alone do not explain (or guarantee) favorable optimization dynamics for large-scale nonconvex training.

{Based on the above discussions,} a natural question arises: can we design memory-efficient optimization algorithms that (a) admit clear nonconvex convergence guarantees while avoiding the systematic bias that can arise from subspace updates, and (b) achieve a strictly improved iteration complexity for finding an $\epsilon$-approximate stationary point?

To address both challenges, We propose \textbf{O}mni-\textbf{M}asked \textbf{G}radient \textbf{D}escent (\textbf{OMGD}) and show that it achieves an improved iteration complexity of $\tilde{\mathcal{O}}(\epsilon^{-3})$ for driving
\(
\min_{0 \le t \le T}\ \bigl\|\nabla F(\theta_t)\bigr\| \le \epsilon
\)
in the nonconvex setting, where $\tilde{\mathcal {O}}$ represents big-$\mathcal {O}$ ignoring logarithmic terms. %Specifically, we leverage the random-reshuffling sampling strategy, which generates a fresh random permutation of the dataset at each epoch and is
Specifically, we adopt the random reshuffling sampling strategy, where at the beginning of each epoch, a fresh random permutation of the dataset is generated and traversed without replacement by the gradient descent algorithm.
It is
known to yield faster and more stable convergence than with-replacement sampling in both convex and nonconvex problems~\cite{Nguyen2021UnifiedShuffling,Lu2022ExampleSelection,Qiu2024RRM}. 
% OMGD extends this without-replacement principle from {\color{orange}only} data to both data and parameter coordinates: within each cycle, every (mask, sample) pair is visited exactly once. 

OMGD generalizes the without-replacement principle from data sampling to a unified traversal over joint coordinates arising from the randomized training process.
In this work, we instantiate this principle over data samples and parameter subspaces: within each cycle, every (parameter mask, sample) pair is visited exactly once.
This full-coverage structure allows the gradient errors introduced by masking to cancel out over a cycle, enabling us to exploit the variance-reduction benefits of reshuffling while still saving memory through low-dimensional updates.

This work makes the following contributions:
\begin{itemize}
    \item We propose {OMGD}, an optimization method that couples data reshuffling with coordinate selection, and we provide both nonconvex and convex convergence analysis showing the strictly improved iteration complexity of $\tilde{\mathcal{O}}(\epsilon^{-3})$ and $\tilde{\mathcal{O}}(\epsilon^{-1})$ for reaching an $\epsilon$-approximate stationary point.
    \item We provide a mechanism-level explanation for why popular memory-efficient designs (e.g., {LISA} and {GoLore}) cannot inherit the improved convergence rates and may revert to worse rates, which is analyzed in detail through an illustrative example in Section~\ref{subsec:illustrative}.
    \item We show that OMGD can be seamlessly integrated into most existing optimizers. In particular, by applying OMGD to \textsc{LISA}, we obtain \textsc{LISA-wor}, which achieves strong performance against competitive baselines in fine-tuning both vision transformers (ViT)~\cite{wu2020visual} and language models (RoBERTa)~\cite{roberta2019}, as well as in pre-training GPT-2~\cite{radford2019language}, under tight memory budgets.
\end{itemize}
% \fcolorbox{black}{gray!40}{\text{The suffix \textbf{wor} means \textbf{w}ith\textbf{o}ut \textbf{r}eplacement in OMGD.}}

% \begin{tcolorbox}[
%   colback=light gray,
%   colframe=black,
%   arc=3pt,
%   boxrule=0.5pt,
%   left=6pt,right=6pt,top=6pt,bottom=6pt
% ]
% Omni-Mask: If we illustrate correctly above in the paragraph, maybe this box can be removed...
% \end{tcolorbox}

\begin{table}[t]
\centering
\caption{Comparison of methods by iteration complexity.}
\label{tab:1}
\resizebox{0.9\linewidth}{!}{
\begin{tabular}{l|cc}
\toprule
\multirow{2}{*}{Algorithm} & \multicolumn{2}{c}{Iter. complexity} \\
% \cmidrule{2-3}
 & Nonconvex & Convex \\
\midrule
SGD \citep{Ghadimi2013Nonconvex} & $\mathcal{O}(\epsilon^{-4})$ & $\mathcal{O}(\epsilon^{-2})$ \\
RR-SGD \citep{Nguyen2021UnifiedShuffling} & $\tilde{\mathcal{O}}(\epsilon^{-3})$ & $\tilde{\mathcal{O}}(\epsilon^{-1})$ \\
RRM-SGD \citep{Qiu2024RRM} & $\tilde{\mathcal{O}}(\epsilon^{-3})$ & - \\
GoLore \citep{He2024GoLore} & $\mathcal{O}(\epsilon^{-4})$ & - \\
LISA \citep{pan2024lisalayerwiseimportancesampling} & - & $\mathcal{O}(\epsilon^{-2})$ \\
\textbf{OMGD(Ours)} & \textbf{$\tilde{\mathcal{O}}(\epsilon^{-3})$} & \textbf{$\tilde{\mathcal{O}}(\epsilon^{-1})$} \\
\bottomrule
\end{tabular}}
\vspace{-0.5cm}
\end{table}

\section{Related Work}

Our work is primarily linked to two lines of prior research.

\textbf{Memory-efficient Training.} Memory-efficient training aims to reduce GPU memory usage while preserving the benefits of full-parameter fine-tuning. Within PEFT, adapter-based approaches such as LoRA~\cite{Hu2021LoRA} and QLoRA~\cite{Dettmers2023QLoRA} add small low-rank or quantized modules on top of a frozen backbone. Sparse or selective PEFT methods instead update part of the original parameters: SIFT ~\cite{song2024sparsefinetuningpretrainedlarge} performs gradient-based component sparsification to modify only a small fraction of weights, and LISA ~\cite{pan2024lisalayerwiseimportancesampling} samples a subset of important layers to update  while freezing the others. Beyond PEFT, {gradient and optimizer compression} methods reduce the cost of storing or propagating optimization states. Low-rank projection schemes such as GaLore~\cite{Zhao2024GaLore} and GoLore~\cite{He2024GoLore} compress gradients into a low-dimensional subspace to shrink optimizer states, while other work studies gradient-dropout regularization in meta-learning~\cite{tseng2020regularizing}, gradient-based parameter selection~\cite{li2025enhancinglargelanguagemodel} and dropped backpropagation in fine-tuning (e.g., DropBP~\cite{woo2024dropbp}). These studies show that many components of the training process can be modified to save memory, but most are analyzed mainly empirically or under strong assumptions, and their interaction with mini-batch sampling strategies remains largely unexplored.

\textbf{SGD with Random Reshuffling.} {Random reshuffling (RR)} is a permutation-based method of sampling mini-batches: at each epoch the dataset is randomly permuted and then processed sequentially. 
This scheme is the default in many deep learning libraries because it leads to stable training and often converges faster than sampling with replacement in large-scale applications~\cite{bottou2010large,goodfellow2016deep}. On the theoretical side, recent work has shown that random reshuffling can improve convergence rates compared with i.i.d.\ sampling and has developed unified analyses for shuffling-type methods in both convex and nonconvex settings~\cite{Nguyen2021UnifiedShuffling,Lu2022ExampleSelection}. Momentum variants under random reshuffling also enjoy iteration-complexity bounds and last-iterate guarantees~\cite{Qiu2024RRM}.  These results suggest that the mini-batch construction and data-ordering scheme can have a substantial impact on the convergence behavior of stochastic-gradient methods.

\section{Methodology}

\begin{figure*}
    \centering
    \includegraphics[width=0.9\textwidth]{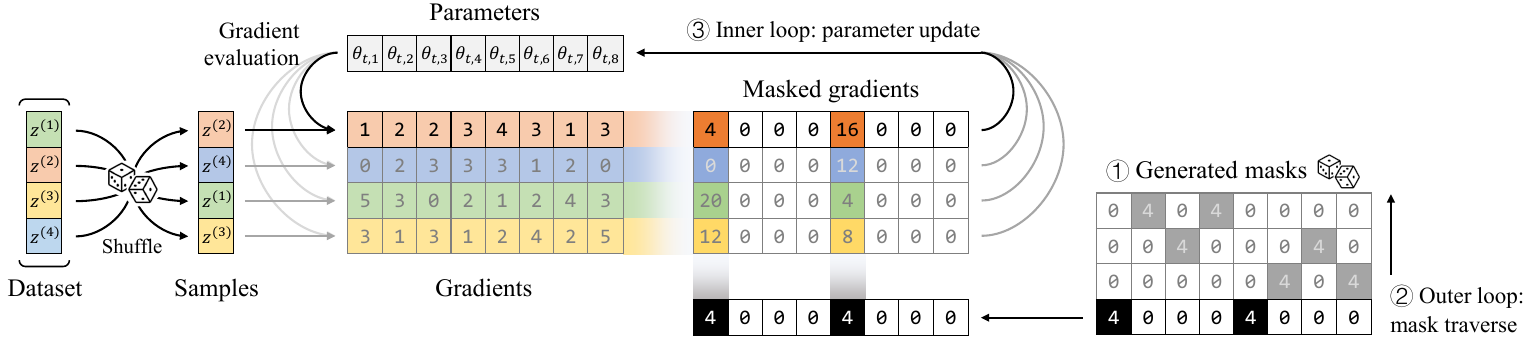}
    \caption{Illustration of the epochwise mask application in OMGD  with $d=8,M=4,N=4$.  
\ding{172} \textbf{Generated masks} satisfy the condition given in \eqref{mask_requirement}.  
\ding{173} \textbf{Outer loop} processes the $M$ masks sequentially, corresponding to $M$ consecutive epochs within one cycle.  
\ding{174} \textbf{Inner loop} performs a full dataset pass for each mask, computing the masked gradient defined in \eqref{mask_gradient_form} to update the model parameters.}
    \label{fig:flowchart}
    \vspace{-0.25cm}
\end{figure*}

\subsection{Problem Formulation}
Throughout this paper, we consider the following empirical risk minimization (ERM) problem:
\begin{equation}\label{ERM_problem}
    \min_{\theta\in\mathbb{R}^d} F(\theta) = \frac{1}{N}\sum_{i=1}^{N} f_{}(\theta;z^{(i)}),
\end{equation}
where $\{z^{(i)}\}_{i=1}^{N}$ is the fixed sample set and $F, f$ are evaluation functions from $\mathbb{R}^{d}$ to $\mathbb{R}$. We update the parameter $\theta$ via the standard SGD:
\begin{equation}\label{sgd_general_form}
    \theta_{t+1}=\theta_{t}-\eta_t g_t,
\end{equation}
where $\eta_t$ denotes the learning rate, and $g_t$ is a stochastic gradient which serves as an estimator of $\nabla F(\theta_t)$. It is common practice to take
\begin{equation*}
 g_t=\nabla f(\theta_t;\mathcal{B}_t)= \frac{1}{|\mathcal{B}_t|}\sum_{i\in\mathcal{B}_t}\nabla f(\theta_t;z^{(i)}),  
\end{equation*}
here $\mathcal{B}_t$ is the mini-batch sampled at step $t$. For simplicity of presentation, we focus on the case where only one data point will be sampled at step $t$, denoted as $z_t$, i.e., the mini-batch size of $\mathcal{B}_t$ is 1. Accordingly, $g_t=\nabla f(\theta_t;z_t)$. 
For quick reference, we list commonly used notations in Table \ref{tab:notation_list}.

\subsection{Omni-Masked Gradient Descent}

We now introduce {OMGD}, which is formally stated in Algorithm~\ref{alg:1}. 
Specifically, $M$ masks $\{S^{(j)}\}_{j=1}^{M}$ are generated at the beginning of each cycle and satisfy
\begin{equation}\label{mask_requirement}
\setlength\abovedisplayskip{5pt}%shrink space
\setlength\belowdisplayskip{5pt}
\sum_{j=1}^{M} S^{(j)} = M\mathbf{1}_d,
\end{equation}
where $\mathbf{1}_d\in\mathbb{R}^d$ denotes the all-ones vector. OMGD generates a random order $\mathrm{RandomPermutation}([M]\times[N])$ within cycle $k$, and iterates through the $MN$ pairs. Let the number of iterations $T$ be divisible by $MN$. At global step $t$, the algorithm uses the masked stochastic gradient
\begin{equation}\label{mask_gradient_form}
 g_t \;=\; S^{(j)} \odot \nabla f\big(\theta_t; z^{(i)}\big),   
\end{equation}
and applies the update $\theta_{t+1}=\theta_t-\eta_t g_t$. As a result, after one cycle, the sequence of pairs $\{(S_t,z_t)\}_{t=(k-1)MN}^{kMN-1}$ visits every element of $\{z^{(i)}\}_{i=1}^N\times\{S^{(j)}\}_{j=1}^{M}$ exactly once. This procedure is visualized in Figure~\ref{fig:flowchart}, which depicts its epochwise implementation where each mask is applied sequentially across training epochs.

In \eqref{mask_requirement}, the specific factor $M$ is not essential, since any scalar multiple can be absorbed into the learning rate $\eta_t$; what matters is that $\sum_{j=1}^{M} S^{(j)}$ is a scalar multiple of $\mathbf{1}_d$, which ensures that the aggregate masked updates over a cycle provide balanced coverage across coordinates. Moreover, the without-replacement traversal leads to more uniform parameter updates and reduces the average gradient error through cyclewise cancellation and variance-reduction effects.

% Figure~\ref{fig:flowchart} provides a concrete illustration with $\theta\in\mathbb{R}^4$, $N=8$, and $M=4$, where the four masks are chosen as
% \begin{equation*}
% \begin{aligned}
%     S^{(1)}=(4,0,0,0),\quad
% S^{(2)}=(0,4,0,0), \\
% S^{(3)}=(0,0,4,0),\quad
% S^{(4)}=(0,0,0,4).
% \end{aligned}
% \end{equation*}
% In this example, each update selects exactly one coordinate of $\theta$, and one cycle consists of $MN=32$ iterations that traverse all pairs in $[4]\times[8]$ specified by $\mathcal{R}_k$.

\begin{algorithm}[t]
\caption{Omni-Masked Gradient Descent}
\label{alg:1}
\begin{algorithmic}[1]
\STATE \textbf{Input}: Dataset $\{z^{(i)}\}_{i=1}^N$; learning rates $\{\eta_t\}_{t\ge0}$; number of masks per cycle $M$; number of iterations $T$;
\STATE \textbf{Initialize}: Model parameter $\theta_0$.
\FOR{$k = 1, \cdots, \frac{T}{MN}$}  
\STATE Generate the mask set $\{S^{(j)}\}_{j=1}^{M}$ for this cycle.
\STATE Generate a \textbf{mask\_sample\_order} \\$\mathcal{R}_k \gets \mathrm{RandomPermutation}\big([M]\times[N]\big).$
\FOR{$\ell = 1, \cdots, MN$} 
\STATE $(j,i) \gets\mathcal{R}_k(\ell)$
\STATE $t \gets (k-1)\,MN + (\ell-1)$
\STATE $g_t \gets S^{(j)} \odot \nabla f\big(\theta_t; z^{(i)}\big)$
\STATE $\theta_{t+1} \gets \theta_t - \eta_t\, g_t$
\ENDFOR
\ENDFOR
\STATE \textbf{Output}: Final parameter $\theta_{T}$.
\end{algorithmic}
\end{algorithm}

\section{Theoretical Results}\label{sec:theory}

\begin{table}[t]
\caption{Notation list.}
\vspace{-0.1cm}
\centering
\label{tab:notation_list}
\resizebox{\linewidth}{!}{
\begin{tabular}{l|l|l|l}
\toprule
$z^{(i)}$ & $i$ th data point & $z_t$ & Sampled data point at step $t$ \\
\midrule
$S^{(j)}$ & $j$ th mask & $S_t$ & Sampled mask at step $t$ \\
\midrule
$\times$ & Cartesian product & $[M]$ & $\{1,\dots,M\}$ \\
\midrule
$\odot$ & Hadamard product & $\lceil\cdot\rceil$ & Ceiling function \\
\bottomrule
\end{tabular}
}
\vspace{-0.3cm}
\end{table}

In this section, we demonstrate the theoretical superiority of Algorithm \ref{alg:1}. We derive its convergence rate in both convex and nonconvex settings. Specifically, we consider solving the ERM problem \eqref{ERM_problem} via SGD \eqref{sgd_general_form}, where the total sample set is denoted as $\mathcal{Z}:=\{z^{(i)}\}_{i=1}^N$ and the stochastic gradient is taken as $g_t=S_t\odot\nabla f(\theta_t;z_t).$

Our theoretical results are based on the following assumptions, where Assumption \ref{low_boundedness}-\ref{l-smoothness} are standard for the analysis of convergence rate, and  Assumption \ref{uniformly_control} specifies that the error of mini-batch gradient is bounded by the real one. Unless otherwise stated, $\|\cdot\|$ denotes the Euclidean norm.
\begin{assumption}[Lower boundedness]\label{low_boundedness}
    The optimization objective $F:\mathbb{R}^d\rightarrow\mathbb{R}$ satisfies $\inf_{\theta\in\mathbb{R}^d}F(\theta)>-\infty$.
\end{assumption}
\begin{assumption}[$L$-smoothness]\label{l-smoothness}
    There exists $L>0$, such that for any $\theta,\theta'\in\mathbb{R}^d$ and any sample $z\in\mathcal{Z}$,
    \begin{equation*}
        \|\nabla f(\theta;z)-\nabla f(\theta';z)\|\le L\|\theta-\theta'\|.
    \end{equation*}
\end{assumption}

\begin{assumption}\label{uniformly_control}
    For all samples $z\in\mathcal{Z}$ and $\theta\in\mathbb{R}^d$, there exist $C_1,C_2\ge 0$ such that $$\|\nabla f(\theta;z)-\nabla F(\theta)\|^2\le C_1^2+C_2^2\|\nabla F(\theta)\|^2.$$  
\end{assumption}
Then we can derive an important lemma, suggesting that the average gradient error can be bounded by terms that do not depend on the length $m$ of the summation of sequences. Let $\{\eta_t\}$ be non-increasing and non-negative step size sequence.

\begin{lemma}\label{key lemma}
    If Assumptions \ref{low_boundedness}-\ref{uniformly_control} hold, then for any $\tau\ge0$ and $m>0$, the sequence $\{(S_t,z_t)\}$ generated by Algorithm~\ref{alg:1} satisfies
    % Suppose that the {Assumptions \ref{low_boundedness}-\ref{uniformly_control}} hold. $\{(S_t,z_t)\}$ are generated by Algorithm~\ref{alg:1}. Let $\{\eta_t\}$ be non-increasing and non-negative sequence, then for any $\tau\ge0$ and $m>0$, we have
    \begin{equation}\label{mean_error_effect}
    \begin{aligned}
        \left\|\sum_{t=\tau}^{\tau+m-1}\eta_t\left(S_t\odot\nabla f(\theta_\tau;z_t)-\nabla F(\theta_\tau)\right)\right\|^2\\\le\eta_\tau^2\left(C^2+\Phi^2\|\nabla F(\theta_\tau)\|^2\right),
    \end{aligned}
    \end{equation}
    where the constants $C,\Phi$ depend on $C_1,C_2,M,N$.
\end{lemma}

By applying Lemma~\ref{key lemma}, which bounds the cumulative gradient error independently of the window length, we obtain the following result characterizing the evolution of Algorithm~\ref{alg:1} over 
$m$ consecutive steps.

\begin{lemma}[Descent lemma]\label{descent lem}
    %Let $L$ be the Lipschitz constant in Assumption~\ref{l-smoothness} and $C,\Phi$ be the constants in Lemma~ \ref{key lemma}. 
    For any $\tau\ge0$, under Assumptions \ref{low_boundedness}-\ref{uniformly_control}, let $m>0$ be the integer that satisfies $3\eta_{\tau}\Phi\le \gamma\le \frac{1}{6LM}$, where $\gamma = \sum_{t=\tau}^{\tau+m-1}\eta_t$. Then we have
    \begin{equation*}
        F(\theta_{\tau+m}) \le F(\theta_{\tau})-\frac{\gamma}{4}\|\nabla F(\theta_\tau)\|^2+\frac{2}{\gamma}\eta_{\tau}^2C^2.
    \end{equation*}
\end{lemma}

We now leverage Lemma~\ref{descent lem} to establish the main result. 

\begin{theorem}[Convergence rate, nonconvex case]\label{main thm_nonconvex}
%Suppose that 
If {Assumptions \ref{low_boundedness}-\ref{uniformly_control}} hold, then for any  accuracy $\epsilon>0$, SGD \eqref{sgd_general_form} with constant step size $\eta=\left({6LM
\left(\left\lceil\frac{4C}{\epsilon}\right\rceil+\left\lceil3\Phi\right\rceil\right)}\right)^{-1}$ yields that the number of iterations $T$ needed to achieve $\min_{0\le t\le T} \|\nabla F(\theta_t)\|\le \epsilon$ is at most
$$T = \left\lceil\frac{48\Delta LM}{\epsilon^2}\right\rceil\left(\left\lceil\frac{4C}{\epsilon}\right\rceil+\left\lceil3\Phi\right\rceil\right)=\mathcal {O}(\epsilon^{-3}),$$
where $\Delta:=F(\theta_0)-\inf_{\theta} F(\theta)$.
\end{theorem}

We next impose a stronger regularity assumption, namely the $\mu$-Polyak-Łojasiewicz (PL) condition, which is strictly weaker than $\mu$-strong convexity and can hold for certain nonconvex objectives. 

\begin{assumption}[$\mu$-PL condition]\label{pl_condition}
There exists a (not necessarily unique) minimizer $\theta^*$ of $F$. Define $F^*:=\min_{\theta}F(\theta)=F(\theta^*)$. $F$ satisfies the PL inequality, i.e.,
% $F$ satisfies $\mu$-PL condition, i.e., the minimizer $\theta^*$ exists (need not be unique), such that $F^*:=\min_{\theta}F(\theta)=F(\theta^*)$, and there holds
\begin{equation*}
\frac{1}{2}\|\nabla F(\theta)\|^2\ge \mu(F(\theta)-F^*).    
\end{equation*}
\end{assumption}

With Assumption~\ref{pl_condition} in place, we now state the strengthened convergence guarantee, whose iteration complexity improves to $\tilde{\mathcal{O}}(\epsilon^{-1})$.

\begin{theorem}[Convergence rate, $\mu$-PL condition]\label{main thm_convex}
If {Assumptions \ref{low_boundedness}-\ref{uniformly_control} and \ref{pl_condition}} hold, then for any target accuracy $\epsilon>0$, SGD \eqref{sgd_general_form} with constant step size $\eta=\left({6LM
\left(\left\lceil\sqrt{\frac{8C^2}{\mu\epsilon^2}}\right\rceil+\lceil3\Phi\rceil\right)}\right)^{-1}$ yields that the number of iterations $T$ needed to achieve $F(\theta_T)-F^*\le \epsilon^2$ is at most
$$T = \left(\left\lceil\sqrt{\frac{8C^2}{\mu\epsilon^2}}\right\rceil+\lceil3\Phi\rceil\right)\left\lceil\frac{12LM}{\mu}\log\frac{2\Delta}{\epsilon^2}\right\rceil = \tilde{\mathcal{O}}(\epsilon^{-1}).$$

\end{theorem}

In addition to establishing the convergence rate under a constant learning rate, we also analyze the convergence behavior under a diminishing step size schedule. The corresponding results are presented in Theorems~\ref{thm:nonconvex_diminishing_step} and \ref{thm:convex_diminishing_step}.

It is worth noting that %Lemma \ref{key lemma} does not hold for other compressors which are independent of mini-batch sampling, 
Lemma~\ref{key lemma} does not hold for other compressors whose gradient compression mappings are generated in an i.i.d. manner and are independent of mini-batch sampling,
which constitutes a key reason why i.i.d. compressors do not exhibit the improved convergence rate induced by RR. We formulate this argument in Proposition \ref{prop: iid_compressor}.

\begin{proposition}\label{prop: iid_compressor}
Suppose the sequence $\{z_t\}$ is sampled by RR, and the mask $S_t$ is i.i.d. generated across $t$ with every
entry satisfying that $(S_t)_j\sim \frac{1}{r}\cdot \operatorname{Bernoulli}(r)$ for $j=1,\dots,d$. If {Assumptions \ref{low_boundedness}-\ref{uniformly_control}} hold, then for $\tau=0,N,2N,3N,\dots$ and any $m>0$, we have
% Suppose that the {Assumptions \ref{low_boundedness}-\ref{uniformly_control}} hold. The sequence $\{z_t\}$ is sampled using RR. The mask $S_t$ is i.i.d. generated across $t$ and every
% entry of $S_t$ satisfies that $(S_t)_j\sim \frac{1}{r}\cdot \operatorname{Bernoulli}(r)$ for $j=1,\dots,d$. Denote the sequence of learning rates as $\{\eta_t\}$, then for $\tau=0,N,2N,3N,\dots$ and any $m>0$, we have
    \begin{equation}\label{violate_mean_error_effect}
    \begin{aligned}
    \mathbb{E}\left\|\sum_{t=\tau}^{\tau+m-1}\eta_t\left(S_t\odot\nabla f(\theta_\tau;z_t)-\nabla F(\theta_\tau)\right)\right\|^2\\\ge\left(\sum_{t=\tau}^{\tau+m-1}\eta_t^2\right)\frac{1-r}{r}\mathbb{E}\|\nabla F(\theta_\tau)\|^2.
    \end{aligned}
    \end{equation}
\end{proposition}

\begin{remark}\label{remark:notaion_of_mask}
    If we take $r \in (0,1)$ such that $rd$ is an integer, and sample $S_t \sim \operatorname{Uniform}\big(\{ S \in \{0, 1/r\}^d \mid \|S\|_0 = rd \}\big)$, where $\|\cdot\|_0$ denotes the $L^0$ norm, which counts the number of non-zero elements. Then 
    \vspace{-0.2cm}
    \begin{itemize}
        \item $S_t$ satisfies the condition of Proposition~\ref{prop: iid_compressor};
        \vspace{-0.1cm}
        \item $r$ represents the \emph{keep ratio}, i.e., the proportion of coordinates that are updated. 
    \end{itemize}
\end{remark}

\begin{remark}
    To ensure the mask in Algorithm~\ref{alg:1} has a comparable keep ratio, we can constrain the number of non‑zero entries in each $S^{(j)}$ to be at most $rd$, which leads to the choice $M = \lceil 1/r \rceil$. Specifically, these masks satisfy $S^{(j)} \in \{ S \in \{0, M\}^d \mid \|S\|_0 = rd \}$ for $j = 1,\dots,M-1$, and $S^{(M)} \in \{ S \in \{0, M\}^d \mid \|S\|_0 = d - (M-1)rd \}$.
\end{remark}
\vspace{0.2cm}
\begin{remark}
    The expectation $\mathbb{E}$ in Proposition \ref{prop: iid_compressor} is taken over the randomness of parameter $\theta_t$, sampled data point $z_t$ and mask $S_t$.
\end{remark}

\section{Experiments}

% {\color{red}Yanghui: add a sentence here...}
%在这个section中，我们展示两个什么实验，并用without-replacement (wor) suffix标记集成了OMGD实例/算法/方法，whatever名字。 总之加个总述性的话，空间不够可以把更多公式变成行内
To comprehensively evaluate the effectiveness of OMGD, this section conducts systematic experiments including a mechanism-level illustrative example (Section~\ref{subsec:illustrative}), image classification tasks (Section~\ref{subsec:image}), RoBERTa fine-tuning (Section~\ref{subsec:roberta}), and GPT-2 pre-training (Section~\ref{subsec:llm}); throughout, we append the suffix \texttt{wor} (without-replacement) to denote any method that integrates OMGD.

\subsection{Illustrative Example}
\label{subsec:illustrative}

\begin{figure*}[h]
\centering
\includegraphics[trim=0cm 0.3cm 0cm 0cm, width=0.85\textwidth]{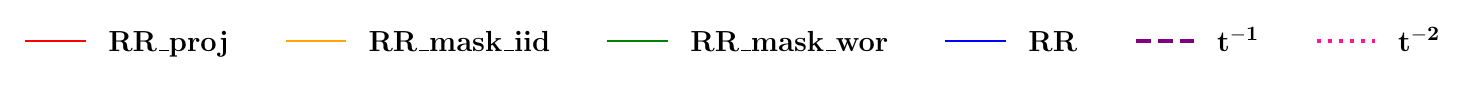}
\subfloat[$\|\theta_t-\theta^*\|^2$.]{
\label{fig:lse_a}
\includegraphics[trim=0cm 0.5cm 0cm 0cm, width=0.235\textwidth]{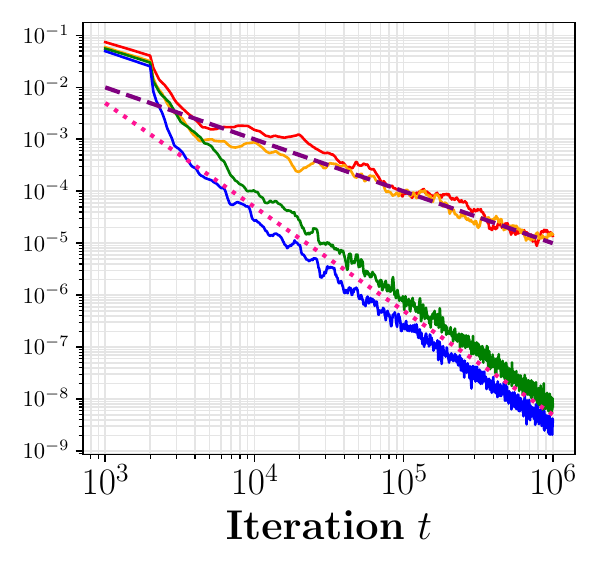}}\hfill
\subfloat[Decay term.]{
\includegraphics[trim=0cm 0.5cm 0cm 0cm, width=0.235\textwidth]{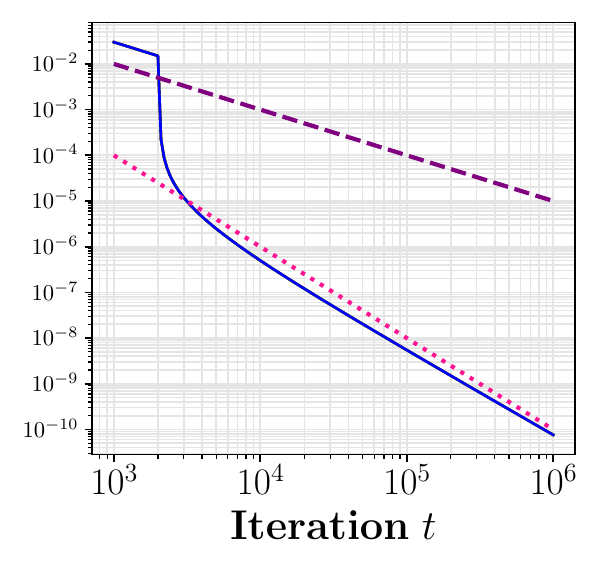}}\hfill
\subfloat[Data-reshuffle term.]{
\includegraphics[trim=0cm 0.5cm 0cm 0cm, width=0.235\textwidth]{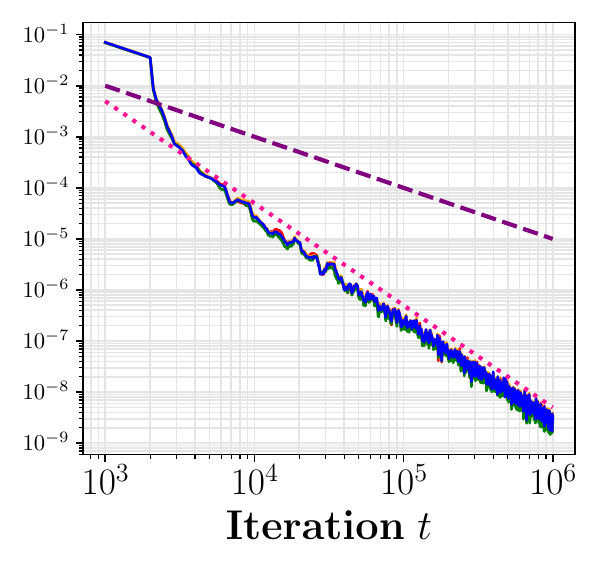}}\hfill
\subfloat[Compression-error term.]{
\includegraphics[trim=0cm 0.5cm 0cm 0cm, width=0.235\textwidth]{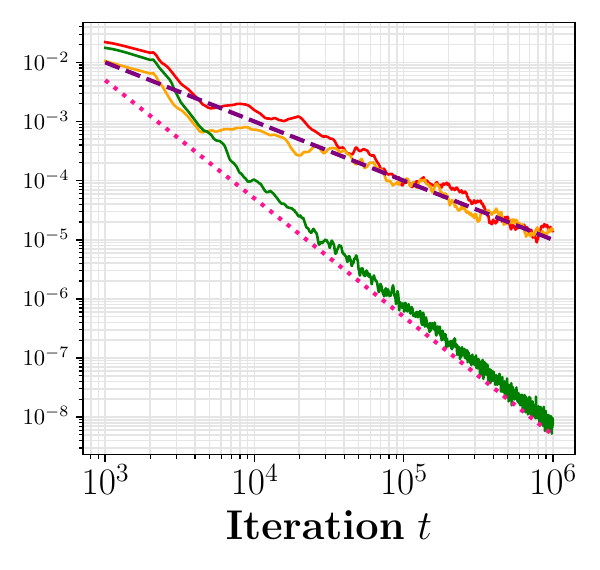}}
\caption{Squared $L^2$ norm of overall error, %parameter distance, 
decay term,  data-reshuffle term, and compression-error term.}
\vspace{-0.3cm}
\label{fig:lse}
\end{figure*}

In this subsection, we present a simple example to illustrate that \textbf{injecting i.i.d. mask or low-rank projection into gradients would incur bad performance} in terms of convergence rate. Specifically, we consider stochastic gradients of the form $g_t = \mathcal{C}_t\big(\nabla f(\theta_t; z_t)\big)$,
% \[
% g_t = \mathcal{C}_t\big(\nabla f(\theta_t; z_t)\big),
% \]
where compression mappings $\{\mathcal{C}_t\}_{t\ge 0}$ are independently resampled across iterations. This formulation captures i.i.d.\ masking as in {LISA}~\cite{pan2024lisalayerwiseimportancesampling}, and i.i.d.\ low-rank random projections as in {GoLore}~\cite{He2024GoLore}.

Denote the sample set as $\{z^{(i)}\}_{i=1}^n=\{(x^{(i)},y^{(i)})\}_{i=1}^n$, where the feature vector $x^{(i)}\in\mathbb{R}^d$ and the label $y^{(i)}\in\mathbb{R}$. We consider the general regression problem: 
\begin{equation*}
\min_{\theta\in\mathbb{R}^d} F(\theta)=\frac{1}{n}\sum_{i=1}^{n}f\big(\theta; x^{(i)}, y^{(i)}\big),
\end{equation*}
and take $f\big(\theta; x^{(i)}, y^{(i)}\big)=\big((x^{(i)})^\top\theta - y^{(i)}\big)^{2}$ as the  example, 
% \begin{equation*}
% F(\theta)=\frac{1}{n}\sum_{i=1}^{n}f\big(\theta; x^{(i)}, y^{(i)}\big)=\frac{1}{n}\sum_{i=1}^{n}\big((x^{(i)})^\top\theta - y^{(i)}\big)^{2}.   
% \end{equation*}
%We rewrite $F$ as 
i.e., $F(\theta)= \frac{1}{2}\theta^\top A\theta-b^\top\theta+c$, where $A =\frac{2}{n}\sum_{i=1}^{n} x^{(i)}(x^{(i)})^{\top},\;
b=\frac{2}{n}\sum_{i=1}^{n} x^{(i)} y^{(i)}$.
% \begin{equation*}
% A =\frac{2}{n}\sum_{i=1}^{n} x^{(i)}(x^{(i)})^{\top},
% b=\frac{2}{n}\sum_{i=1}^{n} x^{(i)} y^{(i)},c=\frac{1}{n}\sum_{i=1}^{n} (y^{(i)})^2.
% \end{equation*}
Let $(x_t,y_t)$ be the data point sampled at step $t$, then $\nabla f(\theta_t;x_t,y_t)=2x_t(x_t^\top\theta_t-y_t)$. The stochastic gradients $g_t$ are taken as:

\begin{itemize}
\item\label{synthetic_wor_sample}  \textbf{RR.} $g_t =\nabla f(\theta_t; x_t, y_t)$, where $(x_t, y_t)$ is sampled from $\{(x^{(i)}, y^{(i)})\}_{i=1}^{n}$ under random reshuffling.
\item\label{synthetic_wor_sample_wor_mask}  \textbf{RR\_mask\_wor (Ours).} $g_t = S_t \odot \nabla f(\theta_t; x_t, y_t)$ is the proposed Omni-Masked gradients in Algorithm~\ref{alg:1}. 
\item\label{synthetic_wor_sample_iid_mask}  \textbf{RR\_mask\_iid.} $g_t = S_t \odot \nabla f(\theta_t; x_t, y_t)$, where $S_t$ is an i.i.d.\ mask vector and $(x_t, y_t)$ is sampled from $\{(x^{(i)}, y^{(i)})\}_{i=1}^{n}$ under random reshuffling.
\item\label{synthetic_wor_sample_iid_proj}  \textbf{RR\_proj.} $g_t = \frac{1}{r}P_tP_t^\top \nabla f(\theta_t; x_t, y_t)$, where $P_t$ is independently sampled from the uniform distribution on a Stiefel manifold $\text{St}_{d,rd}$ (see Remark \ref{remark:stiefel_manifold}) and $(x_t, y_t)$ is sampled from $\{(x^{(i)}, y^{(i)})\}_{i=1}^{n}$ under random reshuffling. This structure is similar to that in Golore.
\end{itemize}

\begin{remark}
    The construction of the mask $S_t$, whether drawn i.i.d. or via without‑replacement sampling, is consistent with the formalism provided in Remark~\ref{remark:notaion_of_mask}.
\end{remark}

\begin{remark}\label{remark:stiefel_manifold}
The Stiefel manifold is defined as $$\text{St}_{m_1,m_2}=\{P\in\mathbb{R}^{m_1\times m_2}:P^\top P=I_{m_2}\}.$$ A random matrix uniformly distributed on $\text{St}_{m_1,m_2}$ can be realized by $Z(Z^\top Z)^{-1/2}$, where the elements of $Z\in\mathbb{R}^{m_1\times m_2}$ are i.i.d. sampled from $\mathcal{N}(0,1)$\cite{chikuse2012statistics}.
\end{remark}

Let $\lambda_{\max},\lambda_{\min}$ be the maximum and minimum eigenvalue of the positive definite matrix $A$. Denote $\theta^*:=A^{-1}b$ as the optimal value and $\rho_t:=\mathbb{E}\|\theta_{t}-\theta^*\|^2$ as the estimation error. We analyze the convergence behavior of $\theta_t$ through $\rho_t$. Note that $\|\nabla F(\theta_t)\|^2=\|A(\theta_t-\theta^*)\|^2\in\left[\lambda_{\min}^2\|\theta_{t}-\theta^*\|^2,\; \lambda_{\max}^2\|\theta_{t}-\theta^*\|^2\right]$,
% \begin{equation*}
%     \|\nabla F(\theta_t)\|^2=\|A(\theta_t-\theta^*)\|^2\le\lambda_{\max}^2\|\theta_{t}-\theta^*\|^2,
% \end{equation*}
which implies that the convergence rate of $\rho_t$ directly determines the iteration complexity for achieving $\mathbb{E}\|\nabla F(\theta_t)\|^2\le\epsilon^2$. Specifically, if $\rho_t = \mathcal{O}(t^{-2})$, then $\mathbb{E}\|\nabla F(\theta_t)\|^2$ reaches $\epsilon^2$ within $\mathcal{O}(\epsilon^{-1})$ iterations; if $\rho_t = \Omega(t^{-1})$, it requires at least $\Omega(\epsilon^{-2})$ iterations. %Therefore, discussing the convergence rate of $\rho_t$ is equivalent to discussing the corresponding iteration complexity. 
Therefore, analyzing the convergence rate of $\rho_t$ is equivalent to characterizing the corresponding iteration complexity.
To present results that align directly with the convergence plots shown in Figure~\ref{fig:lse}, we focus on the convergence rate instead of iteration complexity in this section.

The following theorem show that, both RR-SGD and our method achieve the sharp convergence rate of $\mathcal{O}(t^{-2})$.

\begin{theorem}\label{thm:upper_bound_wor_sample_wor_mask}
If the stochastic gradient takes the form of RR\_mask\_wor or RR; the learning rates $\{\eta_t\}$ satisfy $\frac{c_0}{t}\le\eta_t \le\frac{c_1}{t}$, $|\eta_t-\eta_{t+1}|\le \mathcal {O}(t^{-2})$ for large enough $t$; and $c_0\lambda_{\min}>2$, then we have $\rho_t\le \mathcal {O}(t^{-2})$.
% If the stochastic gradient takes the form of RR\_mask\_wor or RR, $\{\eta_t\}$ are the learning rates that satisfy $\frac{c_0}{t}\le\eta_t \le\frac{c_1}{t}$, $|\eta_t-\eta_{t+1}|\le \mathcal {O}(t^{-2})$ for large enough $t$, $\alpha:=c_0\lambda_{\min}>2$, then we have $\rho_t\le \mathcal {O}(t^{-2})$. 
\end{theorem}

In contrast to Theorem~\ref{thm:upper_bound_wor_sample_wor_mask}, with i.i.d.\ masking or low-rank projection, $\rho_t$ admits a lower bound of order $\Omega(t^{-1})$.

\begin{theorem}\label{thm:lower_bound_wor_sample_iid_mask_or_proj}
If the stochastic gradient takes the form of RR\_mask\_iid or RR\_proj; the learning rates $\{\eta_t\}$ satisfy $\frac{c_0}{t}\le\eta_t \le\frac{c_1}{t}$, $|\eta_t-\eta_{t+1}|\le \mathcal {O}(t^{-2})$ for large enough $t$; $c_0\lambda_{\min}>\frac{1}{2}$; and $\rho_t\le \mathcal {O}(t^{-1})$, then we have $\rho_t=\Omega(t^{-1})$.
% If the stochastic gradient takes the form of RR\_mask\_iid or RR\_proj, $\{\eta_t\}$ are the learning rates that satisfy $\frac{c_0}{t}\le\eta_t \le\frac{c_1}{t}$, $|\eta_t-\eta_{t+1}|\le \mathcal {O}(t^{-2})$ for large enough $t$, $\alpha:=c_0\lambda_{\min}>\frac{1}{2}$,  and $0<\rho_t\le \mathcal {O}(t^{-1})$, then we have $\rho_t=\Omega(t^{-1})$. 
\end{theorem}

\paragraph{Insight behind the lower bound in Theorem~\ref{thm:lower_bound_wor_sample_iid_mask_or_proj}.}

The lower bound of $\Omega(t^{-1})$ stems from the fact that i.i.d.\ masking or low-rank projection introduces an order-independent perturbation at each iteration. 

The SGD update \eqref{sgd_general_form} decomposes as
\begin{align*}
&\theta_{t+1}-\theta^*
=\underbrace{\Big(\prod_{u=0}^{t}(I-\eta_uA)\Big)(\theta_0-\theta^*)}_{\text{decay term}}
\\
&+\underbrace{\sum_{u=0}^{t}\Big(\prod_{i=u}^{t-1}(I-\eta_iA)\Big)\eta_u
\left(
\nabla F(\theta_u)-\nabla f(\theta_u;x_u,y_u)\right)}_{\text{data-reshuffle term}}
\\&+\underbrace{\sum_{u=0}^{t}\Big(\prod_{i=u}^{t-1}(I-\eta_iA)\Big)\eta_u
\left(
\nabla f(\theta_u;x_u,y_u)-g_u\right)}_{\text{compression-error term}}.
\end{align*}

In this decomposition, the \textbf{decay term} is common to all methods and generally decays faster than $\mathcal{O}(t^{-1})$. For all RR methods, the \textbf{data‑reshuffle term} is also well‑controlled and decays faster than $\mathcal{O}(t^{-1})$. 

However, the \textbf{compression‑error term}—which stems from i.i.d.\ masking or projection—introduces compression noise that is independent across iterations and of the data ordering. Consequently, for \textbf{RR\_mask\_iid} and \textbf{RR\_proj}, even when without‑replacement sampling reduces the variance in the data‑reshuffle term, the compression‑error term does not benefit from this cancellation; it accumulates iteratively and yields a persistent $\Omega(t^{-1})$ variance component, thereby becoming the dominant part of the lower bound. Under the theorem’s assumption $\rho_t\le\mathcal{O}(t^{-1})$, the upper and lower bounds coincide, establishing tightness.

Figure \ref{fig:lse}(a) shows that \textbf{RR\_mask\_wor} enjoys a faster convergence rate of $\mathcal {O}(t^{-2})$ in terms of $\|\theta_t-\theta^*\|^2$, while \textbf{RR\_mask\_iid} and \textbf{RR\_proj} incur $\mathcal{O}(t^{-1})$. The other three figures plot the curve of decay term, data-reshuffle term and compression-error term, which supports the above analysis. 

\subsection{Image Classification Tasks}\label{subsec:image}

\begin{table*}[t]
    \centering
    % \vspace{-0.5cm}
    \caption{Fine-tuning results on GLUE benchmark using pre-trained RoBERTa-Base.  
    % \textit{GoLore@20\%} uses GaLore  for the first 80\% iterations and GoLore for the rest 20\%. \textit{Full Params.} denotes full-parameter training.
    }
    % \vspace{-2mm}
    \label{tab:GLUE}
    % \vspace{1em}
    \begin{tabular}{l|cccccccc|c}
    \toprule
         \textbf{\small Algorithm} & \textbf{\small CoLA} & \textbf{\small STS-B} & \textbf{\small MRPC} & \textbf{\small RTE} & \textbf{\small SST2} & \textbf{\small MNLI} & \textbf{\small QNLI} & \textbf{\small QQP} & \textbf{\small Avg}\\
         \midrule
         \small AdamW (Full params) & 64.16 & 90.81 & 92.07 & 80.51 & 94.84 & 87.97 & 92.93 & 89.12 & 86.55\\
         \midrule
          % \small GaLore \cite{Zhao2024GaLore} & 61.32 & 90.24 & 92.55 & 77.62 & {94.61} & 86.92 & 92.06 & 90.84 & 85.77\\
          \small GoLore \cite{He2024GoLore} & 62.62 & \textbf{90.49} & {91.95} & {78.70} & {94.72} &  87.33 & 92.35 & 87.83 & 85.75 \\
          \small SIFT \cite{song2024sparsefinetuningpretrainedlarge} & 62.39 & 90.28 & 92.73 & 77.98 & \textbf{95.18} & 87.40 & 92.59 & 88.72 & 85.91\\
          \small LISA \cite{pan2024lisalayerwiseimportancesampling} & 61.76 & 90.19 & {92.25} & 78.34 & {94.50} & 87.54 & 92.68 & {88.77} & 85.75\\
          \midrule
          \small LISA-scale & 61.51 & 90.20 & 91.91 & 76.17 & {94.27} & 87.55 & 92.71 & \textbf{88.81} & 85.39\\
          \small LISA-wor-no-scale & 62.35 & 90.45 & 92.36 & 78.34 & 94.84 & 87.55 & 92.59 & 88.73 & 85.90\\
         \small LISA-wor (Ours, full) & \textbf{62.98} & \textbf{90.49} & \textbf{92.82} & \textbf{79.06} & {94.72} & \textbf{87.72} & \textbf{92.88} & {88.73} & \textbf{86.18}\\
         % \hline
         % \small GoLore@20\% & \multirow{2}{*}{\textbf{62.09}} & \multirow{2}{*} {90.03} & \multirow{2}{*} {\textbf{92.82}} & \multirow{2}{*} {\textbf{79.42}} & \multirow{2}{*} {\textbf{94.84}} & \multirow{2}{*} {86.74} & \multirow{2}{*} {\textbf{92.31}} & \multirow{2}{*} {90.77} & \multirow{2}{*} {86.13}\\
         % \small +MSGD & & & & & & & & &\\
         \bottomrule
    \end{tabular}
    \vspace{-0.3cm}
\end{table*}

In the following subsections, we generalize our Algorithm \ref{alg:1} to SGD with momentum (SGDM) and AdamW. We evaluate our method on CIFAR-10/100~\cite{krizhevsky2009learning} and ImageNet-1K~\cite{deng2009imagenet} datasets. It is noted that  mask partitioning is crucial in practice, so we explore several partition strategies.

\paragraph{Tensorwise-mask.} 
To elaborate the main idea, we fix the sparsity rate at \(r=0.5\) and consider the \textsc{SGDM-iid} mask as a naïve baseline. In this scheme, at the beginning of each epoch, we independently sample a proportion \(r\) of tensors from \verb|model.parameters()| to be updated. The remaining \(1-r\) fraction of parameters are frozen by setting \verb|param.requires_grad = False|. In contrast, our \textsc{SGDM-wor mask} follows the without-replacement partitioning idea in Algorithm~\ref{alg:1}: every two epochs form one cycle. At the start of each cycle, we randomly split \verb|model.parameters()| into two disjoint blocks (each covering approximately \(50\%\) of the parameters). During the first epoch of the cycle, only the tensors in the first block are updated while the others are frozen; during the second epoch, we switch and only update another block. This preserves an effective per-epoch sparsity of \(r=0.5\) while guaranteeing complementary coverage across consecutive epochs (i.e., without replacement) at the tensor level. 

We train ResNet-20 on CIFAR-10/100 for 200 epochs, and ResNet-18 on ImageNet for 100 epochs, both from scratch. The experimental results are reported in Table~\ref{table:resnet_cifar_imagenet}, which shows that the proposed method achieves higher classification accuracy than the i.i.d. mask. 

\begin{table}[h]
\centering
\caption{Classification accuracy (\%) for training ResNet on various image datasets.}
% \vspace{0.2cm}
\label{table:resnet_cifar_imagenet}
\setlength{\tabcolsep}{6pt}
\renewcommand{\arraystretch}{1.15}
\begin{adjustbox}{max width=\linewidth}
\begin{tabular}{l|ccc}
\toprule
\textbf{Algorithm} & \textbf{CIFAR-10} & \textbf{CIFAR-100} & \textbf{ImageNet} \\
\midrule
SGDM (full params)  & 92.15 & 66.76 & 69.14 \\
\midrule
SGDM-iid mask & 90.80 & 65.99 & 64.06  \\
% \rowcolor{blue!8} 
SGDM-wor mask (Ours)  & \textbf{91.41} & \textbf{66.15} & \textbf{65.34} \\
\bottomrule
\end{tabular}
\end{adjustbox}
\vspace{-0.3cm}
\end{table}

\paragraph{Layerwise-mask.} 
\textsc{LISA}~\cite{pan2024lisalayerwiseimportancesampling} is a training schedule that periodically restricts optimization to a small, randomly chosen subset of \emph{middle layers} while always updating the embedding layer and the classification head.
Formally, let \(N_L\) be the total number of middle layers, \(T\) the total number of optimization steps, \(K\) the \emph{sampling period}, and \(\gamma\) the number of \emph{sampled layers} per period.
Training proceeds in \(\lfloor T/K \rfloor\) periods indexed by \(i=0,\dots,\lfloor T/K \rfloor-1\).
At the beginning of each period \(i\), LISA freezes all layers except the embedding and head, then randomly  selects \(\gamma\) middle layers to unfreeze.
AdamW is then run for \(K\) consecutive epochs, updating only the unfrozen set (embedding, head, and the sampled \(\gamma\) middle layers).
After \(K\) epochs, a new period begins and a fresh random subset of \(\gamma\) middle layers is drawn.

We propose \textsc{LISA-wor}  which seamlessly integrates Algorithm \ref{alg:1} with \textsc{LISA}  by introducing a \textbf{without-replacement policy over middle layers}. There are two modifications compared to \textsc{LISA}, which are highlighted in red in Algorithm \ref{alg:lisa_wor}. First, we maintain a pool of middle layers and partition it into disjoint subsets; within a cycle, the subset of unfrozen middle layers selected at each sampling step is guaranteed to be non-overlapping with the subsets previously activated in the same cycle, until all middle layers have been covered. After exhausting the pool, we reshuffle and start a new cycle. Second, to satisfy the requirement \eqref{mask_requirement}, we rescale the gradients of the selected middle layers by a factor of $N_L/\gamma$. To illustrate this construction, consider a parameter vector $\theta \in \mathbb{R}^6$. Suppose we aim to update $\theta$ in a sparse manner while ensuring that the first and last coordinates are updated at every iteration. We can, for example, choose the masks as 
\begin{equation*}
\begin{aligned}
     S^{(1)} = (1,4,0,0,0,1)^\top, \quad
    S^{(2)} = (1,0,4,0,0,1)^\top,  \\
    S^{(3)} = (1,0,0,4,0,1)^\top, \quad
    S^{(4)} = (1,0,0,0,4,1)^\top.
\end{aligned}
\end{equation*}
With this choice, we have $$\sum_{j=1}^4 S^{(j)} =(4,4,4,4,4,4)^\top= 4\cdot\mathbf{1}_6.$$ Consequently, \textsc{LISA-wor} strictly fulfills the requirements of Algorithm~\ref{alg:1}, hence Lemma~\ref{key lemma} and the convergence guarantee in Theorem~\ref{main thm_nonconvex} follows. The proposed \textsc{LISA-wor} preserves the original design (embedding/head layers remain active) while ensuring complementary coverage of middle layers across consecutive sampling steps, thereby reducing redundancy and improving classification accuracy.

To evaluate the proposed algorithm \textsc{LISA-wor}, we fine-tune the \verb|vit-base-patch16-224-in21k| model~\cite{wu2020visual}, which is pre-trained on ImageNet-21K. We conduct fine-tuning experiments on CIFAR-10/100 for 100 epochs and ImageNet-1K for 10 epochs. Table \ref{table:vit_fine-tuning} shows that \textsc{LISA-wor} beats other benchmarks.

\begin{algorithm}[h]
\caption{\textsc{LISA}/\colorbox{red!30}{\textsc{LISA-wor}}}
\label{alg:lisa_wor}
\begin{algorithmic}[1]
\REQUIRE number of middle layers $N_L$, number of epochs $T$, sampling period $K$, number of sampled layers $\gamma$, initial learning rate $\eta_0$
\STATE Initialize \textsc{unselected layers} as the list of all middle layer indices (i.e., $[1,2,\dots,N_L]$)
\FOR{$i \leftarrow 0$ to $T/K - 1$}
  \STATE Freeze all layers except the embedding and head layer
  \IF{\# $\textsc{unselected layers} < \gamma$}
    \STATE Reset \textsc{unselected layers} to $[1,2,\dots,N_L]$
  \ENDIF
  \STATE Randomly sample $\gamma$ indices of middle layers from the list \textsc{unselected layers} to unfreeze 
  % \STATE\colorbox{red!20} {Update \textsc{unselected layers}}
  \STATE\colorbox{red!30}{Remove these  indices from \textsc{unselected layers}}
  \STATE\colorbox{red!30}{Scale the gradients of selected layers by $N_L/\gamma$} 
  \STATE Run AdamW for $K$ epochs 
\ENDFOR
\end{algorithmic}
\end{algorithm}

\begin{table}[h]
\centering
\vspace{-0.3cm}
\caption{Classification accuracy (\%) for fine-tuning ViT-base on various image datasets.}
\label{table:vit_fine-tuning}
\setlength{\tabcolsep}{6pt}
\renewcommand{\arraystretch}{1.15}
\begin{adjustbox}{max width=\linewidth}
\begin{tabular}{l|ccc}
\toprule
\textbf{Algorithm} & \textbf{CIFAR-10} & \textbf{CIFAR-100} & \textbf{ImageNet}  \\
\midrule
AdamW (full params)  & 99.11 & 92.27 & 80.62 \\
\midrule
GoLore \cite{He2024GoLore} & 98.90 & 90.92 & 79.85  \\
SIFT \cite{song2024sparsefinetuningpretrainedlarge} & 99.09 & 91.45 &  80.86   \\
LISA \cite{pan2024lisalayerwiseimportancesampling}& {98.94} & 92.20 & {81.41}  \\
% \rowcolor{blue!8} 
LISA-wor (Ours)  & \textbf{99.18} & \textbf{92.42} & \textbf{81.64}  \\
\bottomrule
\end{tabular}
\end{adjustbox}
\vspace{-0.6cm}
\end{table}

\subsection{Fine-tuning Experiments of RoBERTa}\label{subsec:roberta}
We fine-tune the \verb|roberta-base| model \cite{roberta2019} on the GLUE benchmarks \cite{wang2019gluemultitaskbenchmarkanalysis}. We run 30 epochs for each task. Memory consumption of all memory-efficient algorithms is controlled at the same level. To validate the core components of Algorithm~\ref{alg:lisa_wor} (i.e., without-replacement layer sampling and gradient scaling), we introduce two ablation variants: \textsc{LISA-scale}, which incorporates gradient scaling into the original LISA framework, and \textsc{LISA-wor-no-scale}, which applies sampling without replacement to middle layers but removes the scale gradient component. As shown in Table \ref{tab:GLUE}, the \textsc{LISA-wor} algorithm consistently outperforms these ablations as well as other memory-efficient baselines, thereby clearly demonstrating the synergistic importance of integrating both modifications into the algorithm design.

\begin{figure*}[t]
\centering

\begin{minipage}[t]{0.24\textwidth}
  \centering
  \includegraphics[trim=0cm 0.8cm 0cm 0cm, width=\linewidth]{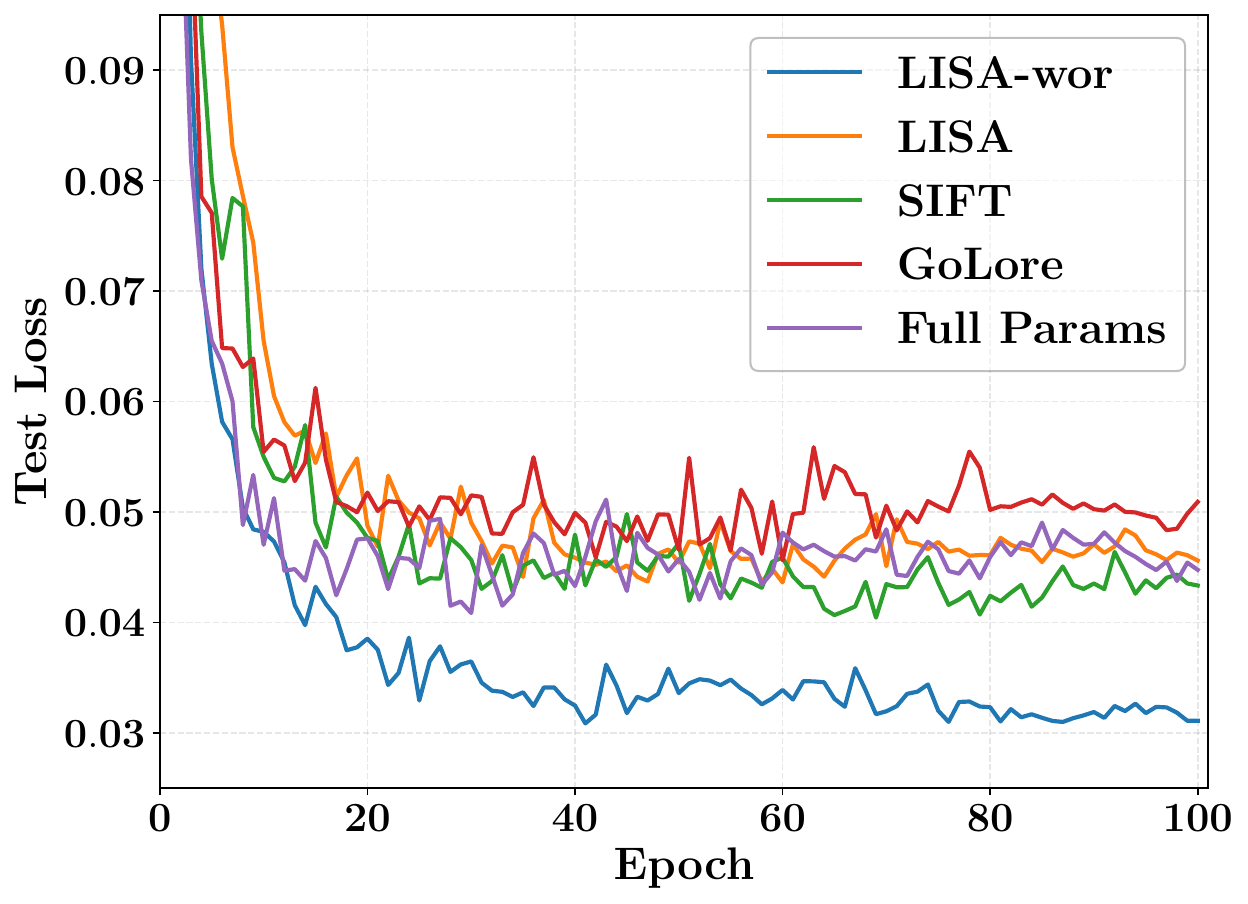}
  \captionof{figure}{Test loss of fine-tuning ViT on CIFAR-10.}
\end{minipage}
\hfill
\begin{minipage}[t]{0.24\textwidth}
  \centering
  \includegraphics[trim=0cm 0.8cm 0cm 0cm, width=0.98\linewidth]{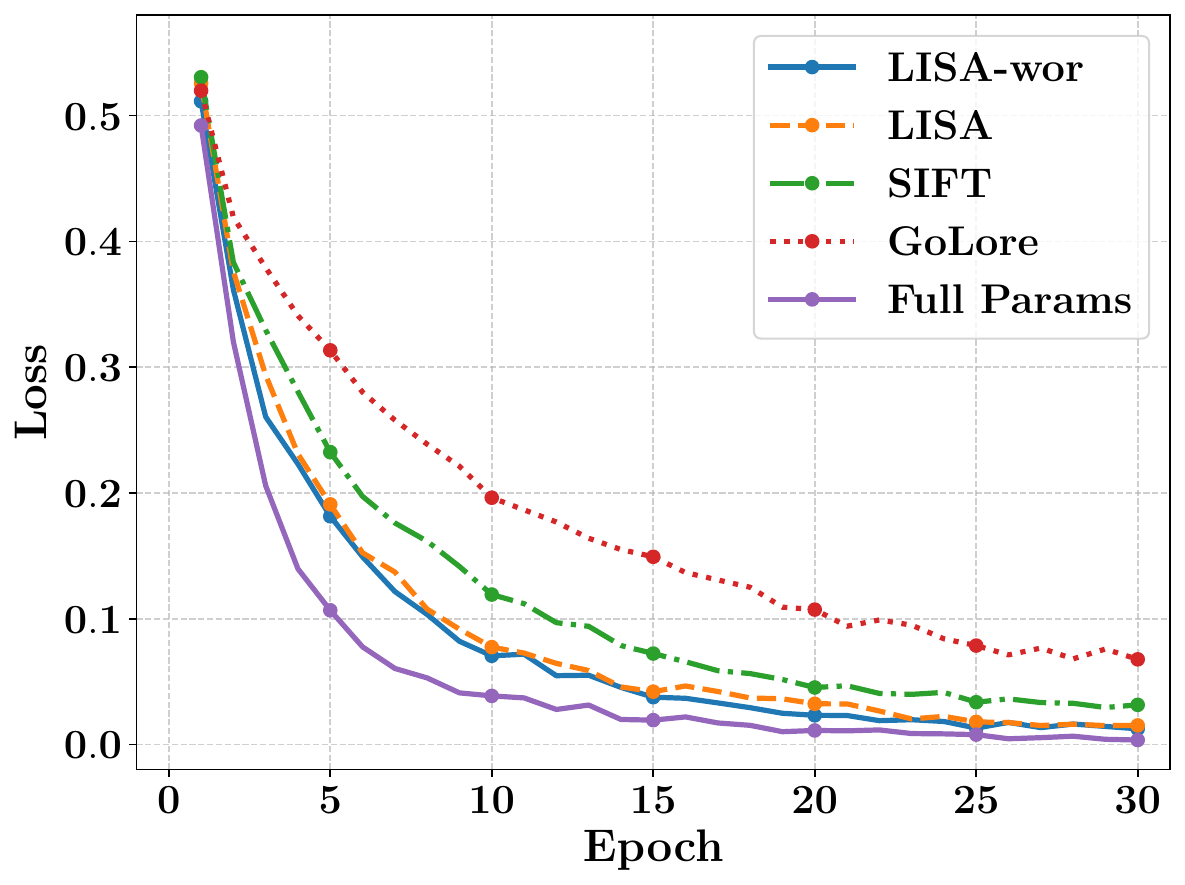}
  \captionof{figure}{Training loss of fine-tuning RoBERTa-Base on CoLA.}
\end{minipage}
\hfill
\begin{minipage}[t]{0.24\textwidth}
  \centering
  \includegraphics[trim=0cm 0.8cm 0cm 0cm, width=0.98\linewidth]{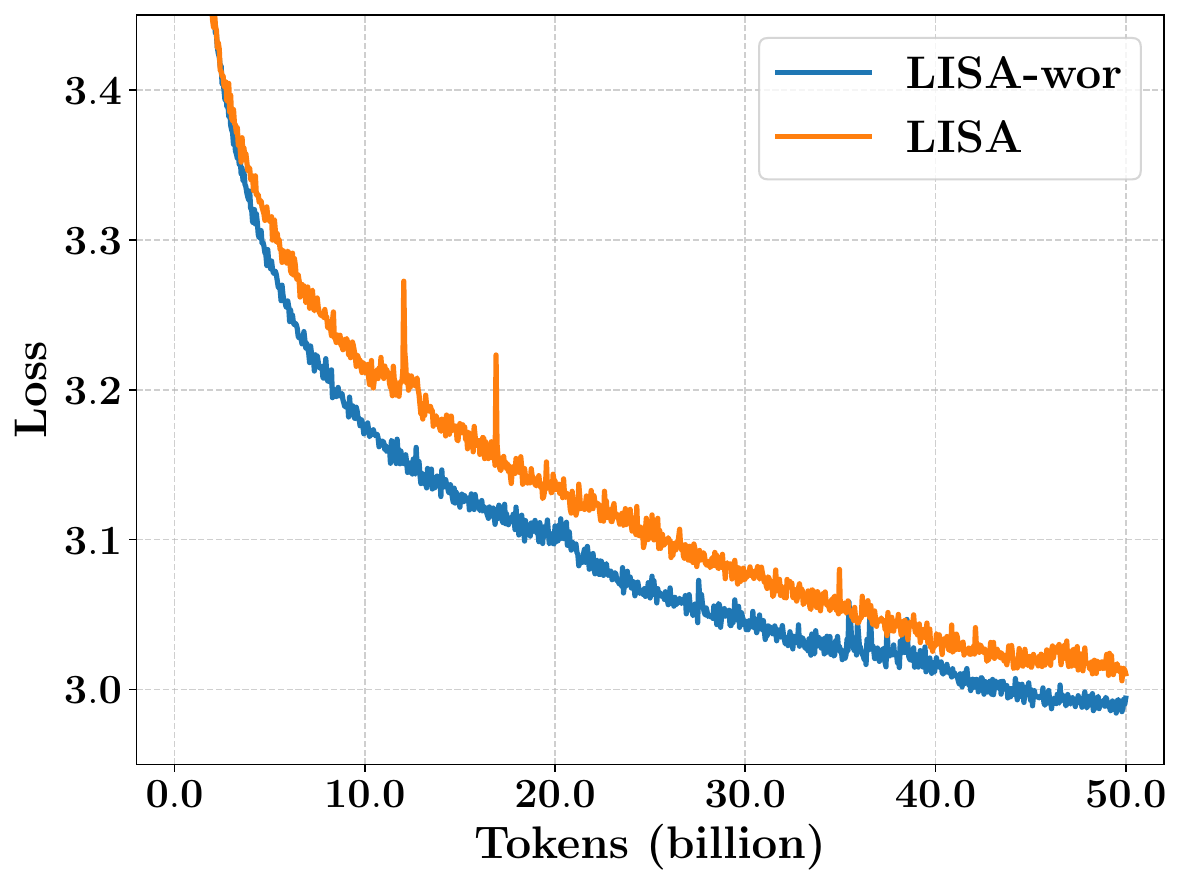}
  \captionof{figure}{Training loss of pre-training GPT-2-124M.}
  \label{fig:pre-train-gpt2-124m}
\end{minipage}
\hfill
\begin{minipage}[t]{0.24\textwidth}
  \centering
  \includegraphics[trim=0cm 0.8cm 0cm 0cm, width=\linewidth]{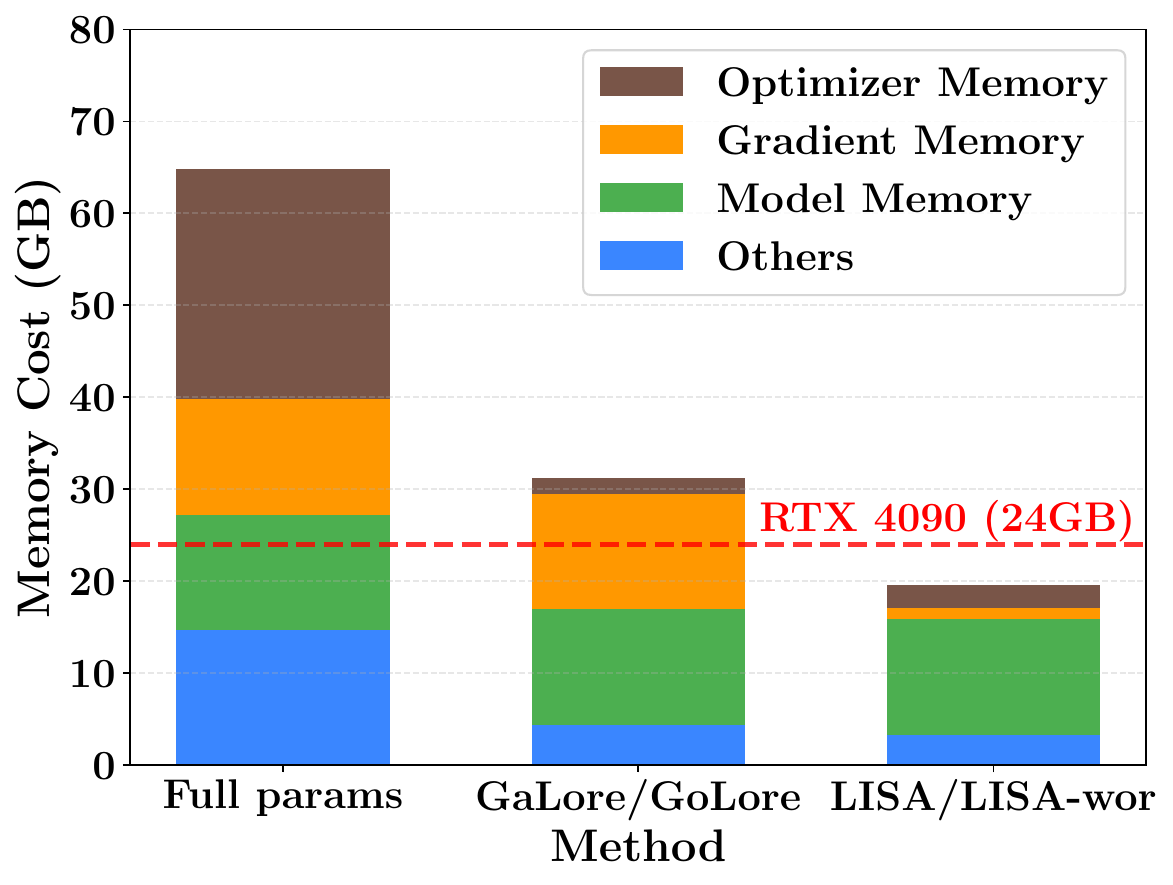}
  \captionof{figure}{GPU memory consumption of LLaMA-7B.}
  \label{fig:memory_consumption}
\end{minipage}

\end{figure*}

\subsection{Pre-training Experiments of LLMs}\label{subsec:llm}

We pre-train GPT-2-124M~\cite{radford2019language} on OpenWebtext~\cite{Gokaslan2019OpenWeb}. The GPT-2-124M model contains 12 middle layers. For both \textsc{LISA} and \textsc{LISA-wor}, we set the sampling layers to 3 and switch the active layers every 100 iterations. As shown in Figure ~\ref{fig:pre-train-gpt2-124m}, \textsc{LISA-wor} outperforms \textsc{LISA}.

We evaluate GPU memory consumption of memory-efficient methods by pre-training a LLaMA-7B model~\cite{touvron2023llama} on C4 dataset~\cite{c4dataset} using a single device, with a total batch size of 512. For {LISA} and \textsc{LISA-wor}, we sample 2 out of 32 middle layers of LLaMA-7B. The results are reported in Figure~\ref{fig:memory_consumption}. \textsc{LISA-wor} reduces total memory consumption by approximately 70\% compared to the full parameter baseline (from 64.86GB to 19.56GB), enabling training on consumer-level GPUs such as the NVIDIA RTX 4090 with 24GB memory. Both {GaLore} and GoLore achieve a 52\% reduction (31.23GB), but fail to reduce gradient memory consumption during backward propagation. This limitation becomes the memory bottleneck, as they maintain the full 12.55GB gradient memory requirement. In contrast, \textsc{LISA-wor} achieves substantial reductions in both gradient memory (to 1.24GB) and optimizer states (to 2.48GB), resulting in superior overall memory efficiency.

\subsection{Ablation studies}

% The two key hyperparameters of \textsc{LISA-wor} are sampling layers $\gamma$ and sampling period $K$ in Algorithm~\ref{alg:lisa_wor}. To obtain intuitive and empirical guidance on these hyperparameter choices, we conduct ablation studies on the fine-tuning of RoBERTa-Base on the CoLA dataset. The configurations for $\gamma$ include $\{1, 2, 3, 4, 6\}$, and $K$ varies from $\{1, 2, 3, 5,6\}$. The findings, presented in Table~\ref{tab:ablation_sampling_period_layers}, reveal that both $\gamma$ and $K$ significantly affect performance. Higher $\gamma$ generally leads to better performance, albeit at the cost of increased memory. Lower $K$ means more frequent layer switching. However, if $K$ is set too small, excessively frequent switching may lead to performance degradation. Generally, the rule of thumb is: Higher $\gamma$ and moderate $K$ lead to better performance. 

The two key hyperparameters of \textsc{LISA-wor} are the number of sampled layers $\gamma$ and the sampling period $K$ in Algorithm~\ref{alg:lisa_wor}.
To provide intuitive empirical guidance on choosing these hyperparameters, we conduct ablation studies by fine-tuning RoBERTa-Base on the GLUE benchmarks.
We sweep $\gamma\in\{1,2,3,4,6\}$ and $K\in\{1,2,3,5,6\}$, and report the results on the CoLA dataset in Table~\ref{tab:ablation_sampling_period_layers}.
Overall, both $\gamma$ and $K$ have a clear impact on performance: larger $\gamma$ generally improves accuracy but increases memory usage, while smaller $K$ corresponds to more frequent layer switching.
When $K$ becomes too small, overly frequent switching can hurt performance.
In practice, these results suggest a simple rule of thumb: use the largest feasible $\gamma$ under the memory budget together with a moderate $K$.

\begin{figure}[h]
\label{fig:fine-tune-cola}
% \centering
\hspace{-0.4cm}
\subfloat[Ablations on $\gamma$.]{
\label{fig:ablations_gamma}
\includegraphics[trim=0cm 0.4cm 0cm 0cm, width=0.5\linewidth]{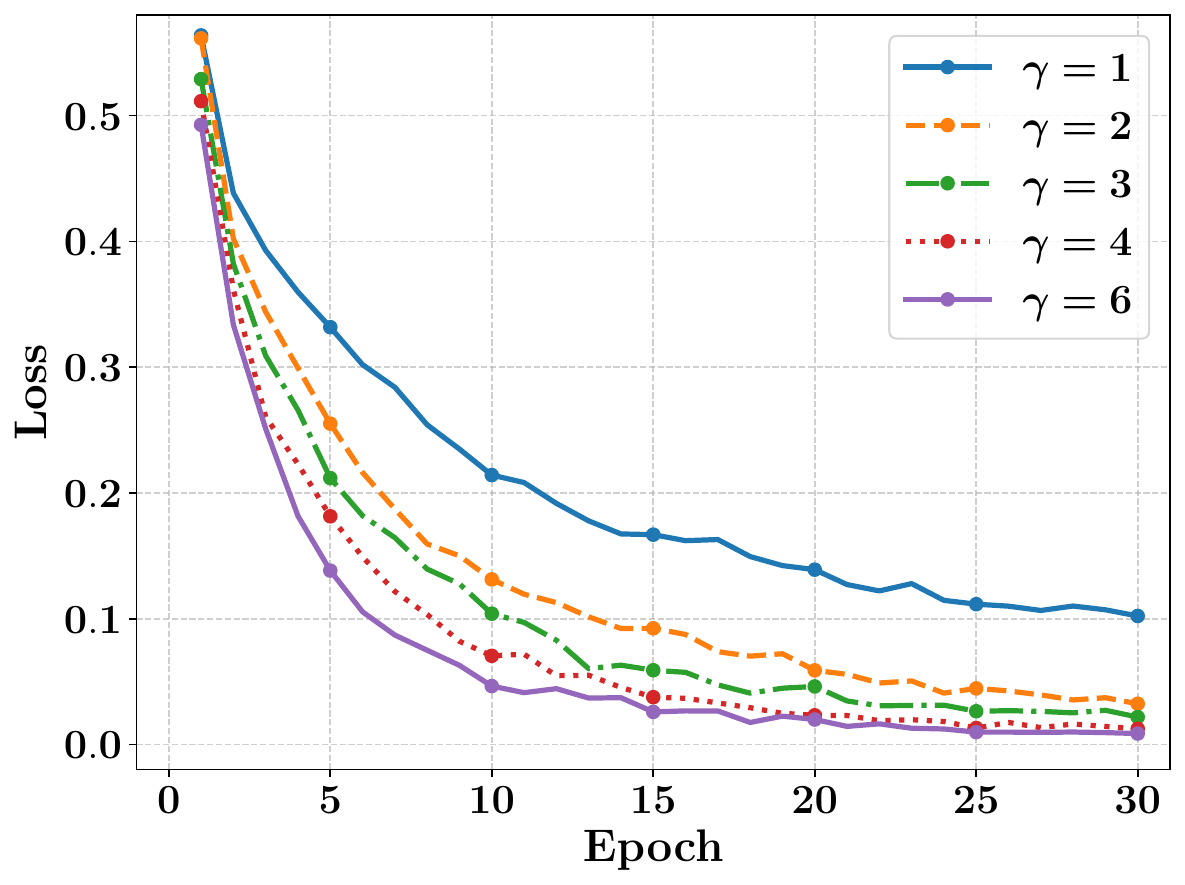}}
\subfloat[Ablations on $K$.]{
\label{fig:ablations_k}
\includegraphics[trim=0cm 0.4cm 0cm 0cm, width=0.5\linewidth]{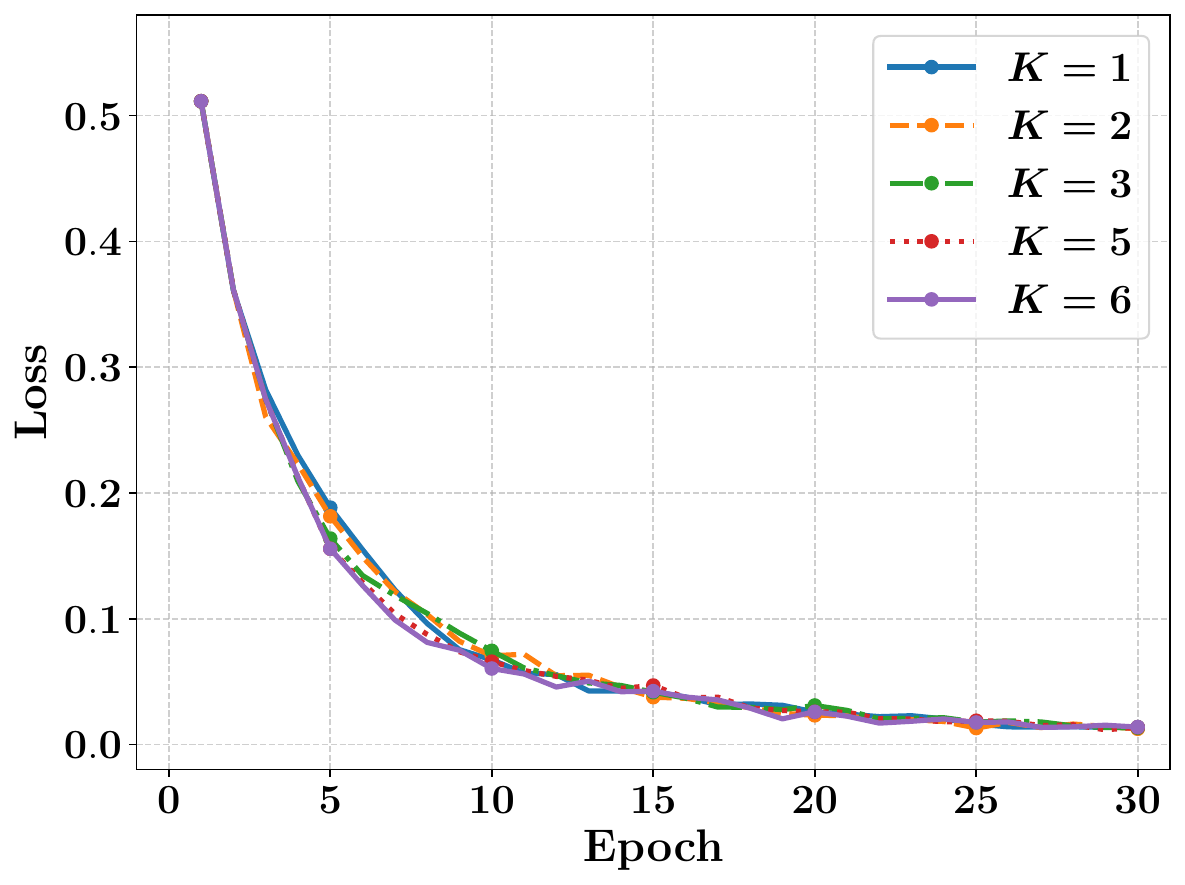}}
\caption{Training loss of fine-tuning RoBERTa-Base on CoLA.}
\vspace{-0.2cm}
\end{figure}

\begin{table}[h]
    \centering
    \caption{Fine-tune RoBERTa-Base on the CoLA dataset, evaluated by Matthews Correlation Coefficient (MCC). Ablations on sampling period $K$ and sampling layers $\gamma$.}
    \label{tab:ablation_sampling_period_layers}
    % \vspace{0.2cm}
    \resizebox{0.4\textwidth}{!}{
    \begin{tabular}{c|ccccc}
    \toprule
    \diagbox{ $\gamma$}{$K$}  & 1 & 2 & 3 & 5 & 6 \\
    \midrule
    1 & 59.81 & 58.12 & 58.46 & 58.46 & 58.71 \\
    2 & 59.89 & 59.57 & 60.82 & 60.09 & 60.33 \\
    3 & 60.63 & 60.76 & 59.81 & 60.75 & 61.48 \\
    4 & 64.16 & 62.98 & 61.69 & 62.33 & 62.84 \\
    6 & 62.60 & 64.47 & 64.50 & 63.11 & 62.43 \\
    \bottomrule
    \end{tabular}
    }
    \vspace{-0.3cm}
\end{table}

\section{Conclusion}

In this paper, we propose \textbf{O}mni-\textbf{M}asked \textbf{G}radient \textbf{D}escent (\textbf{OMGD}), a memory-efficient optimization algorithm that offers significant benefits both in theory and in practice. We provide a rigorous theoretical analysis, demonstrating that OMGD achieves an iteration complexity of at most $\tilde{\mathcal{O}}(\epsilon^{-3})$ for finding $\epsilon$-approximate stationary points under standard nonconvex settings; this represents a substantial improvement over the canonical $\mathcal{O}(\epsilon^{-4})$ bound. Extensive experiments show that OMGD consistently outperforms existing competitive memory-efficient optimization methods across a wide range of training tasks, thereby confirming its practical effectiveness. Moreover, OMGD can serve as a lightweight, plug-and-play method that that can be seamlessly integrated into most mainstream optimizers, effectively reducing the GPU-memory footprint of both gradients and optimizer states. This reduction enhances hardware flexibility and enables both pre-training and fine-tuning under tighter memory constraints on more accessible devices, broadening its applicability in resource-limited settings.

% \section*{Acknowledgements}

% \textbf{Do not} include acknowledgements in the initial version of the paper

% If a paper is accepted, the final camera-ready version can (and usually should)
% include acknowledgements.  Such acknowledgements should be placed at the end of
% the section, in an unnumbered section that does not count towards the paper
% page limit. Typically, this will include thanks to reviewers who gave useful
% comments, to colleagues who contributed to the ideas, and to funding agencies
% and corporate sponsors that provided financial support.

\newpage
\section*{Impact Statement}

% Authors are \textbf{required} to include a statement of the potential broader
% impact of their work, including its ethical aspects and future societal
% consequences. This statement should be in an unnumbered section at the end of
% the paper (co-located with Acknowledgements -- the two may appear in either
% order, but both must be before References), and does not count toward the paper
% page limit. In many cases, where the ethical impacts and expected societal
% implications are those that are well established when advancing the field of
% Machine Learning, substantial discussion is not required, and a simple
% statement such as the following will suffice:

% ``This paper presents work whose goal is to advance the field of Machine
% Learning. There are many potential societal consequences of our work, none
% which we feel must be specifically highlighted here.''

% The above statement can be used verbatim in such cases, but we encourage
% authors to think about whether there is content which does warrant further
% discussion, as this statement will be apparent if the paper is later flagged
% for ethics review.

This paper presents work whose goal is to advance the field of machine learning by developing memory-efficient optimization methods with improved theoretical guarantees. We do not feel any potential societal consequences of this work necessary to be discussed here.


\bibliography{arxiv}
\bibliographystyle{icml2026}

%%%%%%%%%%%%%%%%%%%%%%%%%%%%%%%%%%%%%%%%%%%%%%%%%%%%%%%%%%%%%%%%%%%%%%%%%%%%%%%
%%%%%%%%%%%%%%%%%%%%%%%%%%%%%%%%%%%%%%%%%%%%%%%%%%%%%%%%%%%%%%%%%%%%%%%%%%%%%%%
% APPENDIX
%%%%%%%%%%%%%%%%%%%%%%%%%%%%%%%%%%%%%%%%%%%%%%%%%%%%%%%%%%%%%%%%%%%%%%%%%%%%%%%
%%%%%%%%%%%%%%%%%%%%%%%%%%%%%%%%%%%%%%%%%%%%%%%%%%%%%%%%%%%%%%%%%%%%%%%%%%%%%%%
\newpage
\appendix
\onecolumn

\section{Theoretical Details}

\subsection{Proof of Lemma \ref{key lemma}}

\begin{proof}[proof]
    Let $\beta_t = \eta_t-\eta_{t+1}\ge 0$, for $t = \tau, \dots,\tau+m-2$, and $\beta_{\tau+m-1} = \eta_{\tau+m-1}$, then
    \begin{align*}
         \bigg\| \sum_{t=\tau}^{\tau+m-1}&\eta_t\left(S_t\odot\nabla f(\theta_\tau;z_t)-\nabla F(\theta_\tau)\right)\bigg\|^2 \\
        = &\ \bigg\| \sum_{t=\tau}^{\tau+m-1}\sum_{j=t}^{\tau+m-1} \beta_j \left(S_t\odot\nabla f(\theta_\tau;z_t)-\nabla F(\theta_\tau)\right)\bigg\|^2 
        =  \bigg\| \sum_{j=\tau}^{\tau+m-1}\sum_{t=\tau}^{j} \beta_j \left(S_t\odot\nabla f(\theta_\tau;z_t)-\nabla F(\theta_\tau)\right)\bigg\|^2 \\
        = &\ \eta_\tau^2\bigg\| \sum_{j=\tau}^{\tau+m-1}\frac{\beta_j}{\eta_\tau}
        \sum_{t=\tau}^{j}  \left(S_t\odot\nabla f(\theta_\tau;z_t)-\nabla F(\theta_\tau)\right) \bigg\|^2 
        \le \eta_\tau^2 \sum_{j=\tau}^{\tau+m-1}\frac{\beta_j}{\eta_\tau}\bigg\| 
        \sum_{t=\tau}^{j}  \left(S_t\odot\nabla f(\theta_\tau;z_t)-\nabla F(\theta_\tau)\right) \bigg\|^2,
    \end{align*}
    where the inequality comes from Jensen's inequality by utilizing $\sum_{j=\tau}^{\tau+m-1}\beta_j=\eta_\tau$ and the convexity of $\|\cdot\|^2$. 
    
    Note that at $k$ th cycle, we have $\{(S_t, z_t)\}_{t=kMN}^{(k+1)MN-1}=\{z^{(i)}\}_{i=1}^N\times\{S^{(i)}\}_{i=1}^{M}$, and hence 
    \begin{equation*}
    \sum_{t=kMN}^{(k+1)MN-1}  \left(S_t\odot\nabla f(\theta_\tau;z_t)-\nabla F(\theta_\tau)\right)=0.
    \end{equation*}
    %$\|\sum_{t=\tau}^j(\mathbf{1}-S_t)\|\le C$, thus
    For any indices $\tau\le j$, define integers $k_1,k_2$ be so that $[k_1MN,k_2MN-1]$ is the smallest interval that contains $[\tau,j]$. Then we can obtain the following identity
    \begin{equation*}
     \sum_{t=k_1MN}^{k_2MN-1}  \left(S_t\odot\nabla f(\theta_\tau;z_t)-\nabla F(\theta_\tau)\right)=0,
    \end{equation*}
    and thus we have
    \begin{align*}
        \sum_{t=\tau}^{j}  \left(S_t\odot\nabla f(\theta_\tau;z_t)-\nabla F(\theta_\tau)\right)=- 
        \sum_{t=k_1{MN}}^{\tau-1}  \left(S_t\odot\nabla f(\theta_\tau;z_t)-\nabla F(\theta_\tau)\right) - 
        \sum_{t=j+1}^{k_2MN-1}  \left(S_t\odot\nabla f(\theta_\tau;z_t)-\nabla F(\theta_\tau)\right) .
    \end{align*}
    Since $|\tau-1-k_1MN|\le{MN}$, we obtain that 
    \begin{align*}
    \bigg\| 
        \sum_{t=k_1{MN}}^{\tau-1} & \left(S_t\odot\nabla f(\theta_\tau;z_t)-\nabla F(\theta_\tau)\right) \bigg\|^2\le M^2N^2\sup_t\|S_t\odot\nabla f(\theta_\tau;z_t)-\nabla F(\theta_\tau)\|^2\\
        &\le M^2N^2(\sup_t\|S_t\odot\nabla f(\theta_\tau;z_t)-S_t\odot\nabla F(\theta_\tau)\|^2 + \sup_t\|S_t\odot\nabla F(\theta_\tau)-\nabla F(\theta_\tau)\|^2)\\
        &\le M^3N^2(C_1^2+(C_2^2+1)\|\nabla F(\theta_\tau)\|^2).
    \end{align*}
    Analogously, there holds
    \begin{equation*}
     \left\| 
        \sum_{t=j+1}^{k_2MN-1}  \left(S_t\odot\nabla f(\theta_\tau;z_t)-\nabla F(\theta_\tau)\right) \right\|^2\le M^3N^2(C_1^2+(C_2^2+1)\|\nabla F(\theta_\tau)\|^2).   
    \end{equation*}
    Using the inequality $\|a+b\|^2\le2(\|a\|^2+\|b\|^2)$, we obtain
    \begin{equation*}
    \begin{aligned}
        \bigg\| 
        \sum_{t=\tau}^{j}  & \left(S_t\odot\nabla f(\theta_\tau;z_t)-\nabla F(\theta_\tau)\right) \bigg\|^2\\&\le2\bigg\| 
        \sum_{t=k_1{MN}}^{\tau-1}  \left(S_t\odot\nabla f(\theta_\tau;z_t)-\nabla F(\theta_\tau)\right) \bigg\|^2 + 2\bigg\| 
        \sum_{t=j+1}^{k_2MN-1}  \left(S_t\odot\nabla f(\theta_\tau;z_t)-\nabla F(\theta_\tau)\right) \bigg\|^2. 
    \end{aligned}
    \end{equation*}
    
    Consequently, we can show that with $C=2M^{3/2}NC_1$ and $\Phi=2M^{3/2}N\sqrt{C_2^2+1}$, there holds
    \begin{equation*}
     \bigg\| 
        \sum_{t=\tau}^{j}  \left(S_t\odot\nabla f(\theta_\tau;z_t)-\nabla F(\theta_\tau)\right) \bigg\|^2\le4{M^3N^2}(C_1^2+(C_2^2+1)\|\nabla F(\theta_\tau)\|^2),
    \end{equation*}
    which completes the proof.
\end{proof}

\subsection{Proof of Lemma~\ref{descent lem}}

% By Lemma~\ref{key lemma}, we obtain the following $m$-step descent bound.

% \begin{lemma}[Descent lemma]\label{descent lem}
%     For any $\tau\ge0$, let $m>0$ be the integer that satisfies $3\eta_{\tau}\Phi\le \gamma\le \frac{1}{6L\lceil\frac{1}{r}\rceil}$, where $\gamma = \sum_{t=\tau}^{\tau+m-1}\eta_t$,  then we have
%     \begin{equation*}
%         F(\theta_{\tau+m}) \le F(\theta_{\tau})-\frac{\gamma}{4}\|\nabla F(\theta_\tau)\|^2+\frac{2}{\gamma}\eta_{\tau}^2C^2.
%     \end{equation*}
% \end{lemma}

\begin{proof}
    Let $g = \frac{1}{\gamma}\sum_{t=\tau}^{\tau+m-1}\eta_tS_t\odot\nabla f(\theta_t;z_t)$, then $\theta_{\tau+m} =\theta_{\tau}-\gamma g$. By the $L$-smoothness of $F$, we have
    \begin{align}
        F(\theta_{\tau+m})& = F(\theta_\tau-\gamma g)\nonumber \\
        & \le F(\theta_\tau) - \gamma \nabla F(\theta_\tau)^\top g + \frac{\gamma^2L}{2} \|g\|^2 \nonumber\\
        & = F(\theta_\tau) -\frac{\gamma}{2}\|\nabla F(\theta_\tau)\|^2 - \frac{\gamma}{2}\|g\|^2 + \frac{\gamma}{2}\|g-\nabla F(\theta_\tau)\|^2 +  \frac{\gamma^2L}{2} \|g\|^2 \nonumber\\
        & \le F(\theta_\tau) -\frac{\gamma}{2}\|\nabla F(\theta_\tau)\|^2  + \frac{\gamma}{2}\|g-\nabla F(\theta_\tau)\|^2, \label{eq of L-smooth}
    \end{align}
    since $\gamma L<\frac{1}{6\lceil\frac{1}{r}\rceil}<1$. Then we focus on bounding the third term, i.e., the average gradient error. We get
    \begin{align*}
         \frac{\gamma}{2}\|g-\nabla F(\theta_\tau)\|^2 &= \frac{1}{2\gamma}\|\gamma g-\gamma\nabla F(\theta_\tau)\|^2 = \frac{1}{2\gamma}\bigg\|\sum_{t=\tau}^{\tau+m-1}\eta_t(S_t\odot\nabla f(\theta_t;z_t)-\nabla F(\theta_\tau))\bigg\|^2 \\
        & \le \frac{1}{\gamma}\bigg\|\sum_{t=\tau}^{\tau+m-1}\eta_t S_t\odot(\nabla f(\theta_t;z_t)-\nabla f(\theta_\tau;z_t))\bigg\|^2 + \frac{1}{\gamma}\bigg\|\sum_{t=\tau}^{\tau+m-1}\eta_t\left(S_t\odot\nabla f(\theta_\tau;z_t)-\nabla F(\theta_\tau)\right)\bigg\|^2,
    \end{align*}
    where the last inequality comes from $\|a+b\|^2\le2(\|a\|^2+\|b\|^2)$. By Lemma \ref{key lemma}, we have
    \begin{equation*}
        \frac{1}{\gamma}\bigg\|\sum_{t=\tau}^{\tau+m-1}\eta_t\left(S_t\odot\nabla f(\theta_\tau;z_t)-\nabla F(\theta_\tau)\right)\bigg\|^2\le\frac{1}{\gamma}\eta_\tau^2\left(C^2+\Phi^2\|\nabla F(\theta_\tau)\|^2\right).
    \end{equation*}
    For the first term, we have
    \begin{align*}
        \frac{1}{\gamma}\bigg\|\sum_{t=\tau}^{\tau+m-1} &\eta_t S_t\odot(\nabla f(\theta_t;z_t)-\nabla f(\theta_\tau;z_t))\bigg\|^2 = \gamma\bigg\|\sum_{t=\tau}^{\tau+m-1}\frac{\eta_t}{\gamma} S_t\odot(\nabla f(\theta_t;z_t)-\nabla f(\theta_\tau;z_t))\bigg\|^2 \\
        & \le \gamma  \sum_{t=\tau}^{\tau+m-1}\frac{\eta_t}{\gamma} \left\|S_t\odot(\nabla f(\theta_t;z_t)-\nabla f(\theta_\tau;z_t))\right\|^2 \le \sum_{t=\tau}^{\tau+m-1}\eta_t L^2\left\lceil\frac{1}{r}\right\rceil^2\|\theta_\tau-\theta_t\|^2.
    \end{align*}
    By the update rule of parameters $\theta$, we can obtain
    \begin{align*}
        \sum_{t=\tau}^{\tau+m-1}\eta_t L^2\|\theta_\tau-\theta_t\|^2 =&\ \sum_{t=\tau+1}^{\tau+m-1}\eta_t L^2\bigg\|\sum_{u=\tau}^{t-1}\eta_uS_u
        \odot\nabla f(\theta_u;z_u)\bigg\|^2 \\
         \le&\ 3\sum_{t=\tau+1}^{\tau+m-1}\eta_t L^2 \bigg( \bigg\|\sum_{u=\tau}^{t-1}\eta_uS_u \odot(\nabla f(\theta_u;z_u)-\nabla f(\theta_\tau;z_u))\bigg\|^2 \\
         &\quad\quad\quad\quad\ \ \ \ \ \ \! +\bigg\|\sum_{u=\tau}^{t-1}\eta_u(S_u \odot\nabla f(\theta_\tau;z_u)-\nabla F(\theta_\tau))\bigg\|^2 +\bigg\|\sum_{u=\tau}^{t-1}\eta_u\nabla F(\theta_\tau)\bigg\|^2\bigg).
    \end{align*}
    Here, we apply the %inequality $\|a+b+c\|^2 \le 3(\|a\|^2+\|b\|^2+\|c\|^2)$ 
    {Cauchy–Schwarz inequality} to decompose the squared norm into three parts. We now proceed to bound each of these parts separately:
    \begin{align*}
        &\sum_{t=\tau}^{\tau+m-1}\eta_t L^2\|\theta_\tau-\theta_t\|^2 \\
        % =\sum_{t=\tau+1}^{\tau+m-1}\eta_t L^2\bigg\|\sum_{u=\tau}^{t-1}\eta_uS_u
        % \odot\nabla f(\theta_u;z_u)\bigg\|^2 \\
        % & \le 3\sum_{t=\tau+1}^{\tau+m-1}\eta_t L^2\bigg\|\sum_{u=\tau}^{t-1}\eta_uS_u \odot(\nabla f(\theta_u;z_u)-\nabla f(\theta_\tau;z_u))\bigg\|^2 +3\sum_{t=\tau+1}^{\tau+m-1}\eta_t L^2\bigg\|\sum_{u=\tau}^{t-1}\eta_u(S_u \odot\nabla f(\theta_\tau;z_u)-\nabla F(\theta_\tau))\bigg\|^2 \\
        % & \quad+3\sum_{t=\tau+1}^{\tau+m-1}\eta_t L^2\bigg\|\sum_{u=\tau}^{t-1}\eta_u\nabla F(\theta_\tau)\bigg\|^2 \\
        % &\le 3\sum_{t=\tau+1}^{\tau+m-1}\eta_t L^4\left\lceil\frac{1}{r}\right\rceil^2\bigg(\sum_{u=\tau}^{t-1}\eta_u\bigg)\sum_{u=\tau}^{t-1}\eta_u \|\theta_u-\theta_\tau\|^2 +3\sum_{t=\tau+1}^{\tau+m-1}\eta_t L^2\eta_\tau^2(C^2+\Phi^2\|\nabla F(\theta_\tau)\|^2) \\ 
        % & \quad + 3\sum_{t=\tau+1}^{\tau+m-1}\eta_t L^2\bigg(\sum_{u=\tau}^{t-1}\eta_u\bigg)^2\left\|\nabla F(\theta_\tau)\right\|^2 \\
        &\le 3\sum_{t=\tau+1}^{\tau+m-1}\eta_t \bigg\{L^4\left\lceil\frac{1}{r}\right\rceil^2\bigg(\sum_{u=\tau}^{t-1}\eta_u\bigg)\sum_{u=\tau}^{t-1}\eta_u \|\theta_u-\theta_\tau\|^2 + L^2\eta_\tau^2(C^2+\Phi^2\|\nabla F(\theta_\tau)\|^2) + L^2\bigg(\sum_{u=\tau}^{t-1}\eta_u\bigg)^2\left\|\nabla F(\theta_\tau)\right\|^2 \bigg\} \\
        % &\le 3\sum_{t=\tau+1}^{\tau+m-1}\eta_t L^4\left\lceil\frac{1}{r}\right\rceil^2\left(\sum_{u=\tau}^{t-1}\eta_u\right)\sum_{u=\tau}^{t-1}\eta_u \|\theta_u-\theta_\tau\|^2 +3\sum_{t=\tau+1}^{\tau+m-1}\eta_t L^2\eta_\tau^2(C^2+\Phi^2\|\nabla F(\theta_\tau)\|^2) \\ 
        % & +3\sum_{t=\tau+1}^{\tau+m-1}\eta_t L^2\left(\sum_{u=\tau}^{t-1}\eta_u\right)^2\left\|\nabla F(\theta_\tau)\right\|^2 \\ 重复
        &\le 3\sum_{t=\tau+1}^{\tau+m-1}\eta_t L^4\left\lceil\frac{1}{r}\right\rceil^2\gamma\sum_{u=\tau}^{t-1}\eta_u \|\theta_u-\theta_\tau\|^2+3\gamma L^2\eta_\tau^2(C^2+\Phi^2\|\nabla F(\theta_\tau)\|^2)+3\gamma^3 L^2\left\|\nabla F(\theta_\tau)\right\|^2 \\
        % &=  3\sum_{u=\tau}^{\tau+m-2}\sum_{t=u+1}^{\tau+m-1}\eta_t L^4\left\lceil\frac{1}{r}\right\rceil^2\gamma\eta_u \|\theta_u-\theta_\tau\|^2+3\gamma L^2\eta_\tau^2(C^2+\Phi^2\|\nabla F(\theta_\tau)\|^2)+3\gamma^3 L^2\left\|\nabla F(\theta_\tau)\right\|^2 \\ 重复
        &=  3\sum_{u=\tau}^{\tau+m-2}\sum_{t=u+1}^{\tau+m-1}\eta_t L^4\left\lceil\frac{1}{r}\right\rceil^2\gamma\eta_u \|\theta_u-\theta_\tau\|^2+3\gamma L^2\eta_\tau^2(C^2+\Phi^2\|\nabla F(\theta_\tau)\|^2)+3\gamma^3 L^2\left\|\nabla F(\theta_\tau)\right\|^2 \\
        &\le  3\gamma^2L^2\left\lceil1/r\right\rceil^2\sum_{t=\tau}^{\tau+m-1} L^2\eta_t \|\theta_t-\theta_\tau\|^2+3\gamma L^2\eta_\tau^2(C^2+\Phi^2\|\nabla F(\theta_\tau)\|^2)+3\gamma^3 L^2\left\|\nabla F(\theta_\tau)\right\|^2.
    \end{align*}
    Rearrange the inequality, we have
    \begin{align*}
        \bigg(1-3\gamma^2L^2\left\lceil\frac{1}{r}\right\rceil^2\bigg)\sum_{t=\tau}^{\tau+m-1} L^2\eta_t \|\theta_t-\theta_\tau\|^2\le 3\gamma L^2\eta_\tau^2(C^2+\Phi^2\|\nabla F(\theta_\tau)\|^2)+3\gamma^3 L^2\left\|\nabla F(\theta_\tau)\right\|^2.
    \end{align*}
    Since $3\eta_{\tau}\Phi\le \gamma\le \frac{1}{6L\lceil\frac{1}{r}\rceil}$, we further have
    \begin{align*}
        \left(1-\frac{1}{12}\right)\sum_{t=\tau}^{\tau+m-1} L^2\eta_t \|\theta_t-\theta_\tau\|^2&\le 3\gamma L^2\eta_\tau^2(C^2+\Phi^2\|\nabla F(\theta_\tau)\|^2)+3\gamma^3 L^2\left\|\nabla F(\theta_\tau)\right\|^2 \\
        &\le 3\gamma L^2\eta_\tau^2C^2+\frac{10}{3}\gamma^3 L^2\left\|\nabla F(\theta_\tau)\right\|^2 
        \le 3\gamma L^2\eta_\tau^2C^2+\frac{5}{54\lceil\frac{1}{r}\rceil^2}\gamma\left\|\nabla F(\theta_\tau)\right\|^2,
    \end{align*}
    which gives
    \begin{equation*}
        \sum_{t=\tau}^{\tau+m-1} L^2\eta_t \|\theta_t-\theta_\tau\|^2\le \frac{36}{11}\gamma L^2\eta_\tau^2C^2+\frac{10}{99\lceil\frac{1}{r}\rceil^2}\gamma\left\|\nabla F(\theta_\tau)\right\|^2.
    \end{equation*}
    Last, we bound the RHS of \eqref{eq of L-smooth},
    \begin{align*}
        F(\theta_{\tau+m}) & \le F(\theta_\tau) -\frac{\gamma}{2}\|\nabla F(\theta_\tau)\|^2  + \frac{1}{\gamma}\eta_\tau^2\left(C^2+\Phi^2\|\nabla F(\theta_\tau)\|^2\right)+\frac{36}{11}\gamma L^2\left\lceil\frac{1}{r}\right\rceil^2\eta_\tau^2C^2+\frac{10}{99}\gamma\left\|\nabla F(\theta_\tau)\right\|^2 \\
        &\le F(\theta_\tau) -{\gamma}\left(\frac{1}{2}-\frac{21}{99}\right)\|\nabla F(\theta_\tau)\|^2 + \frac{1}{\gamma}\left(1+\frac{1}{11}\right)\eta_\tau^2C^2 
        \le F(\theta_{\tau})-\frac{\gamma}{4}\|\nabla F(\theta_\tau)\|^2+\frac{2}{\gamma}\eta_{\tau}^2C^2,
    \end{align*}
which completes the proof.
\end{proof}

\subsection{Proof of Theorem \ref{main thm_nonconvex}}
% \begin{theorem}[Nonconvex optimization, constant step]
% Under {Assumptions \ref{low_boundedness}-\ref{uniformly_control}}, for stochastic gradient descend (SGD) with constant step $\eta=\left({6L\left\lceil\frac{1}{r}\right\rceil
%     \left(\left\lceil\frac{4C}{\epsilon}\right\rceil+\left\lceil3\Phi\right\rceil\right)}\right)^{-1}$,  then the number of iteration steps needed to achieve $\min_{0\le t\le T} \|\nabla F(\theta_t)\|\le \epsilon$ is at most
% $$T = \left\lceil\frac{48\Delta L}{\epsilon^2}\left\lceil\frac{1}{r}\right\rceil\right\rceil\left(\left\lceil\frac{4C}{\epsilon}\right\rceil+\left\lceil3\Phi\right\rceil\right)=\mathcal {O}(\epsilon^{-3}).$$
% \end{theorem}

\begin{proof}
    By Lemma \ref{descent lem}, if $m$ is the integer that satisfies $3\eta\Phi\le m\eta\le\frac{1}{6L\left\lceil\frac{1}{r}\right\rceil}$, then for any $\tau\ge0$,
    \begin{align*}
        F(\theta_{\tau+m}) \le F(\theta_{\tau})-\frac{m\eta}{4}\|\nabla F(\theta_\tau)\|^2+\frac{2}{m}\eta{}C^2, \text{ and }
        \|\nabla F(\theta_\tau)\|^2 \le \frac{4}{m\eta }\left(F(\theta_{\tau})-F(\theta_{\tau+m})\right)+\frac{8C^2}{m^2}.
    \end{align*}
    Taking $\tau =0,m,2m,\dots,(K-1)m$ and summing up, where $K$ is the integer to be determined, we have
    \begin{align*}
        \frac{1}{K}\sum_{k=0}^{K-1}\|\nabla F(\theta_{mk})\|^2 \le \frac{4}{m\eta K }\left(F(\theta_{0})-F(\theta_{mK})\right)+\frac{8C^2}{m^2}.
    \end{align*}
    Let $\Delta = F(\theta_{0})-\inf_{\theta}F(\theta)$, and take $m=\left\lceil\frac{4C}{\epsilon}\right\rceil+\left\lceil3\Phi\right\rceil,\eta=\left({6L\left\lceil\frac{1}{r}\right\rceil
    \left(\left\lceil\frac{4C}{\epsilon}\right\rceil+\left\lceil3\Phi\right\rceil\right)}\right)^{-1}, K=\left\lceil\frac{48\Delta L}{\epsilon^2}\left\lceil\frac{1}{r}\right\rceil\right\rceil$, then
    \begin{equation*}
        \frac{4\Delta}{m\eta K } +\frac{8C^2}{m^2}\le \frac{\epsilon^2}{2}+\frac{\epsilon^2}{2} = \epsilon^2,
    \end{equation*}
    thus $\frac{1}{K}\sum_{k=0}^{K-1}\|\nabla F(\theta_{mk})\|^2\le \epsilon^2$, and the number of iterations can be bounded by
    $$mK=\left\lceil\frac{48\Delta L}{\epsilon^2}\left\lceil\frac{1}{r}\right\rceil\right\rceil\left(\left\lceil\frac{4C}{\epsilon}\right\rceil+\left\lceil3\Phi\right\rceil\right)=\mathcal {O}(\epsilon^{-3}),$$
    which completes the proof.
\end{proof}

\subsection{Proof of Theorem \ref{main thm_convex}}

% \begin{theorem}[$\mu$-PL condition, constant step]
% Under {Assumptions \ref{low_boundedness}-\ref{uniformly_control}} and \ref{pl_condition}, for stochastic gradient descend (SGD) with constant step $\eta=\left({6L\left\lceil\frac{1}{r}\right\rceil
%     \left(\lceil\sqrt{\frac{8C^2}{\mu\epsilon^2}}\rceil+\lceil3\Phi\rceil\right)}\right)^{-1}$,  then the number of iteration steps $T$ needed to achieve $F(\theta_T)-F^*\le \epsilon^2$ is at most
% $$T = \left(\left\lceil\sqrt{\frac{8C^2}{\mu\epsilon^2}}\right\rceil+\lceil3\Phi\rceil\right)\left\lceil\frac{12L\lceil1/r\rceil}{\mu}\log\frac{2\Delta}{\epsilon^2}\right\rceil = \tilde{\mathcal{O}}(\epsilon^{-1}).$$
% \end{theorem}

\begin{proof}
Combining Lemma \ref{descent lem} and $\frac{1}{2}\|\nabla F(\theta)\|^2\ge \mu(F(\theta)-F^*)$, we get 
\begin{align*}
F(\theta_{\tau+m}) &\le F(\theta_{\tau})-\frac{m\eta}{4}\|\nabla F(\theta_\tau)\|^2+\frac{2}{m}\eta{}C^2\le F(\theta_{\tau})-\frac{m\eta\mu}{2}( F(\theta_\tau)-F^*)+\frac{2}{m}\eta{}C^2,
\end{align*}
which holds for $3\eta\Phi\le m\eta\le\frac{1}{6L\left\lceil\frac{1}{r}\right\rceil}$. Rearrange the inequality by subtracting $F^*$ on both sides, we have
\begin{equation*}
F(\theta_{\tau+m})-F^* \le (1-\frac{m\eta\mu}{2})( F(\theta_\tau)-F^*)+\frac{2}{m}\eta{}C^2.
\end{equation*}
Taking $\tau=0,m,\dots,(K-1)m$, and applying the above inequality recursively, we have
\begin{align*}
F(\theta_{mK})-F^* \le (1-\frac{m\eta\mu}{2})^K( F(\theta_0)-F^*)+\frac{2}{m}\eta{}C^2\sum_{k=0}^{K-1}(1-\frac{m\eta\mu}{2})^k\le\exp(-\frac{m\eta\mu K}{2})(F(\theta_0)-F^*)+\frac{4C^2}{m^2\mu}. 
\end{align*}
Hence to ensure $F(\theta_{mK})-F^*\le\epsilon^2$, it suffices to take $m=\lceil\sqrt{\frac{8C^2}{\mu\epsilon^2}}\rceil+\lceil3\Phi\rceil$, $\eta=\frac{1}{6L\lceil1/r\rceil m}$ and $K=\lceil\frac{12L\lceil1/r\rceil}{\mu}\log\frac{2\Delta}{\epsilon^2}\rceil$. Consequently, the iteration complexity can be bounded by
\begin{equation*}
T\le mK=\left(\left\lceil\sqrt{\frac{8C^2}{\mu\epsilon^2}}\right\rceil+\lceil3\Phi\rceil\right)\left\lceil\frac{12L\lceil1/r\rceil}{\mu}\log\frac{2\Delta}{\epsilon^2}\right\rceil = \tilde{\mathcal{O}}(\epsilon^{-1}),
\end{equation*}
which completes the proof.
\end{proof}

\subsection{Proof of Proposition \ref{prop: iid_compressor}}

\begin{proof}
Since $S_t$ is independently generated at step $t$, thus $S_t$ is independent of $\{(\theta_k, z_k)\}_{k\le t}$. Thus we have 
\begin{equation*}
\mathbb{E}\left\langle S_i\odot\nabla f(\theta_\tau;z_i)-\nabla f(\theta_\tau;z_i), S_j\odot\nabla f(\theta_\tau;z_j)-\nabla f(\theta_\tau;z_j)\right\rangle=0,\quad\text{for $i>j\ge\tau$.}
\end{equation*}
Note that $\tau=0,N,2N,3N\dots$ and $\theta_\tau\in\sigma(\bigcup_{{k\le \tau-1}}\{z_k,S_k\})$ (the $\sigma$-algebra generated by random variables $z_k,S_k$ up to $\tau-1$), thus $z_t$ is independent of $\theta_\tau$ for $t\ge\tau$ and $\mathbb{E}[\nabla f(\theta_\tau;z_t)|\theta_\tau]= \nabla F(\theta_\tau)$.
Then we can decompose the L.H.S. of \eqref{violate_mean_error_effect} into a sum of $m$ quadratic terms:
\begin{equation*}
    \begin{aligned}
    \mathbb{E}\ &\bigg\|\sum_{t=\tau}^{\tau+m-1}\eta_t\left(S_t\odot\nabla f(\theta_\tau;z_t)-\nabla F(\theta_\tau)\right)\bigg\|^2\\=&\sum_{t=\tau}^{\tau+m-1}\eta_t^2\mathbb{E}\big\|S_t\odot\nabla f(\theta_\tau;z_t)-\nabla F(\theta_\tau)\big\|^2\\
    =&\sum_{t=\tau}^{\tau+m-1}\eta_t^2\left(\mathbb{E}\|S_t\odot\nabla f(\theta_\tau;z_t)\|^2-2\mathbb{E}\left\langle S_t\odot\nabla f(\theta_\tau;z_t),\nabla F(\theta_\tau)\right\rangle+\mathbb{E}\|\nabla F(\theta_\tau)\|^2\right)\\
    =&\sum_{t=\tau}^{\tau+m-1}\eta_t^2\left(\mathbb{E}\|S_t\odot\nabla f(\theta_\tau;z_t)\|^2-\mathbb{E}\|\nabla F(\theta_\tau)\|^2\right)\\
    =&\sum_{t=\tau}^{\tau+m-1}\eta_t^2\left(\frac{1}{r}\mathbb{E}\|\nabla f(\theta_\tau;z_t)\|^2-\mathbb{E}\|\nabla F(\theta_\tau)\|^2\right)
    \\\ge&\sum_{t=\tau}^{\tau+m-1}\eta_t^2\;\frac{1-r}{r}\mathbb{E}\|\nabla F(\theta_\tau)\|^2,
    \end{aligned}
\end{equation*}
where the last inequality comes from Jensen's inequality: 
\begin{equation*}
\mathbb{E}\|\nabla f(\theta_\tau;z_t)\|^2=\mathbb{E}\;\mathbb{E}\big[\|\nabla f(\theta_\tau;z_t)\|^2|\theta_\tau\big]\ge\mathbb{E}\left\|\mathbb{E}[\nabla f(\theta_\tau;z_t)|\theta_\tau]\right\|^2=\mathbb{E}\|\nabla F(\theta_\tau)\|^2.    
\end{equation*}
Thus, the proof is finished.
\end{proof}

\subsection{Convergence rate for diminishing step size}
Next, we analyze the convergence rate under the case of diminishing step. Denote the sequence $\{a_l\}_{l=0}^{\infty}$ satisfies that $a_{l+1}-a_l=m^{(l)}K^{(l)}$ and $a_0=0$. For $a_{l}\le t<a_{l+1}$, we set $\eta_t = \frac{1}{6L\lceil1/r\rceil m^{(l)}}=:\eta^{(l)}$. That is, at step $t\ge 0$, if $t$ lies in $[a_l,a_{l+1})$, we run SGD with constant step $\eta^{(l)}$ for $m^{(l)}K^{(l)}$ number of iterations, where $m^{(l)}$ and $K^{(l)}$ are to be determined.

\begin{theorem}[Nonconvex optimization, diminishing step]\label{thm:nonconvex_diminishing_step}
Under {Assumptions \ref{low_boundedness}-\ref{uniformly_control}}, for stochastic gradient descent (SGD) with diminishing step $\eta_t$,  then the number of iteration steps needed to achieve $\min_{0\le t\le T} \|\nabla F(\theta_t)\|\le \epsilon$ is at most
$$T = \frac{8\,\lceil 3\Phi\rceil}{7}\,
\Bigg(
2+\frac{4\,\Delta}{L\lceil 1/r\rceil}
+\frac{144\,C^2}{\lceil 3\Phi\rceil^2\,\log 4}
\Bigg)^{3/2}\epsilon^{-3}\,
\left(\log\!\Big(\frac{1}{\epsilon}\Big)\right)^{3/2}=\tilde{\mathcal{O}}(\epsilon^{-3}),$$
where $\Delta:=F(\theta_0)-\inf_\theta F(\theta)$.
\end{theorem}

\begin{proof}
Let $m^{(l)}=\lceil3\Phi\rceil 2^l$ and $K^{(l)}= 4^l$, then for $\tau\in[a_l,a_{l+1})$, we have $3\eta^{(l)}\Phi\le m^{(l)}\eta^{(l)}\le\frac{1}{6L\lceil1/r\rceil}$, thus Lemma \ref{descent lem} implies that 
\begin{align*}
F(\theta_{\tau+m^{(l)}}) \le F(\theta_{\tau})-\frac{m^{(l)}\eta^{(l)}}{4}\|\nabla F(\theta_\tau)\|^2+\frac{2}{m^{(l)}}\eta^{(l)}C^2, \\
\|\nabla F(\theta_\tau)\|^2 \le \frac{4}{m^{(l)}\eta ^{(l)}}\left(F(\theta_{\tau})-F(\theta_{\tau+m^{(l)}})\right)+\frac{8C^2}{({m^{(l)}})^2}.
\end{align*}
Taking $\tau = a_l, a_l+m^{(l)}, a_l+2m^{(l)},\dots,a_l+(K^{(l)}-1)m^{(l)}$ and summing up, we obtain
\begin{equation*}
\sum_{k=0}^{K^{(l)}-1}\|\nabla F(\theta_{a_l+m^{(l)}k})\|^2 \le \frac{4}{m^{(l)}\eta^{(l)}  }\left(F(\theta_{a_l})-F(\theta_{a_{l+1}})\right)+\frac{8C^2K^{(l)}}{(m^{(l)})^2}.
\end{equation*}
Summing $l$ from 0 to ${J-1}$ and note that $m^{(l)}\eta ^{(l)}=\frac{1}{6L\lceil1/r\rceil},a_{l+1}-a_l=m^{(l)}K^{(l)}$, then we have
\begin{equation*}
\sum_{l=0}^{J-1}\sum_{k=0}^{K^{(l)}-1}\|\nabla F(\theta_{a_l+m^{(l)}k})\|^2 \le \frac{2}{3L\lceil1/r\rceil}\left(F(\theta_{0})-F(\theta_{a_{J}})\right)+\sum_{l=0}^{J-1}\frac{8C^2K^{(l)}}{(m^{(l)})^2}.
\end{equation*}
Let $T=\sum_{l=0}^{J-1} m^{(l)}K^{(l)}=\sum_{l=0}^{J-1}\lceil3\Phi\rceil 8^l=a_{J}$, then the above inequality gives
\begin{equation*}
\min_{0\le t\le T}\|\nabla F(\theta_{t})\|^2 \le \frac{2}{3L\lceil1/r\rceil\sum_{l=0}^{J-1} K^{(l)}}\left(F(\theta_{0})-F(\theta_{a_{J}})\right)+\frac{1}{\sum_{l=0}^{J-1} K^{(l)}}\sum_{l=0}^{J-1}\frac{8C^2K^{(l)}}{(m^{(l)})^2}.    
\end{equation*}

Since \(\sum_{l=0}^{J-1}K^{(l)}=\sum_{l=0}^{J-1}4^l=(4^{J}-1)/3\) and \(\sum_{l=0}^{J-1}\frac{K^{(l)}}{(m^{(l)})^2}
=\sum_{l=0}^{J-1}\frac{4^l}{\lceil 3\Phi\rceil^2\,4^l}=J/\lceil 3\Phi\rceil^2\), and \(\Delta=F(\theta_0)-\inf_\theta F(\theta)\ge F(\theta_0)- F(\theta_{a_{J}})\), we get the clean bound that
\[
\min_{0\le t\le T}\|\nabla F(\theta_t)\|^2
\le
\frac{2\,\Delta}{L\lceil 1/r\rceil\,(4^{J}-1)}
+\frac{24C^2\,J}{\lceil 3\Phi\rceil^2\,(4^{J}-1)}.
\]
To ensure \(\min_{0\le t\le T}\|\nabla F(\theta_t)\|^2\le \epsilon^2\), it is sufficient to make each term on the right at most \(\epsilon^2/2\), namely
\[
4^{J}\ge 1+\frac{4\,\Delta}{L\lceil 1/r\rceil\,\epsilon^2}
\quad\text{and}\quad
4^{J}\ge 1+\frac{48C^2}{\lceil 3\Phi\rceil^2\,\epsilon^2}\,J.
\]
Fix \(\epsilon\in(0,e^{-1}]\) and choose
\[
J
=
\left\lceil
\log_{4}\!\Big(\frac{1}{\epsilon^2}\Big)
+
\log_{4}\!\Big(
\Big[
2+\frac{4\,\Delta}{L\lceil 1/r\rceil}
+\frac{144\,C^2}{\lceil 3\Phi\rceil^2\,\log 4}
\Big]\,
\log\!\Big(\frac{1}{\epsilon}\Big)
\Big)
\right\rceil.
\]
With this choice,
\[
4^{J}
\ge
\frac{1}{\epsilon^2}\,
\Big[
2+\frac{4\,\Delta}{L\lceil 1/r\rceil}
+\frac{144\,C^2}{\lceil 3\Phi\rceil^2\,\log 4}
\Big]\,
\log\!\Big(\frac{1}{\epsilon}\Big)
\ge
1+\frac{4\,\Delta}{L\lceil 1/r\rceil\,\epsilon^2},
\]
so the first requirement holds. Moreover, we have
\[
J + \frac{\lceil 3\Phi\rceil^2\,\epsilon^2}{48C^2}
\le
\log_{4}\!\Big(\frac{1}{\epsilon^2}\Big)
+
\log_{4}\!\Big(
\Big[
2+\frac{4\,\Delta}{L\lceil 1/r\rceil}
+\frac{144\,C^2}{\lceil 3\Phi\rceil^2\,\log 4}
\Big]\,
\log\!\Big(\frac{1}{\epsilon}\Big)
\Big)
+\frac{\lceil 3\Phi\rceil^2\,\epsilon^2}{48C^2}+1
\le
\frac{3}{\log 4}\,\log\!\Big(\frac{1}{\epsilon}\Big)
\]
for all sufficiently small \(\epsilon\in(0,e^{-1}]\), hence there holds
\[
\frac{48C^2}{\lceil 3\Phi\rceil^2\,\epsilon^2}\,J +1
\le
\frac{144C^2}{\lceil 3\Phi\rceil^2\,\log 4}\,
\frac{\log(1/\epsilon)}{\epsilon^2}
\le
4^J,
\]
so the second requirement also holds. Therefore the above choice of \(J\) guarantees that
\[
\min_{0\le t< T}\|\nabla F(\theta_t)\|^2\le \epsilon^2.
\]

Finally, we translate this \(J\) into an iteration budget.
Using \(T=\lceil 3\Phi\rceil\,(8^{J}-1)/7\) and \(8^{J}=(4^{J})^{3/2}\), we obtain
\[
T
\le
\frac{\lceil 3\Phi\rceil}{7}\,
\Big(4^{J}\Big)^{3/2}
\le
\frac{8\,\lceil 3\Phi\rceil}{7}\,
\Bigg(
2+\frac{4\,\Delta}{L\lceil 1/r\rceil}
+\frac{144\,C^2}{\lceil 3\Phi\rceil^2\,\log 4}
\Bigg)^{3/2}\epsilon^{-3}\,
\left(\log\!\Big(\frac{1}{\epsilon}\Big)\right)^{3/2}.
\]
Thus, we have obtained the desired result.
\end{proof}

\begin{theorem}[$\mu$-PL condition, diminishing step]\label{thm:convex_diminishing_step}
Under {Assumptions \ref{low_boundedness}-\ref{uniformly_control}} and \ref{pl_condition}, for stochastic gradient descent (SGD) with diminishing step $\eta_t$,  then the number of iteration steps $T$ needed to achieve $F(\theta_T)-F^*\le \epsilon^2$ is at most
$$T = \tilde{\mathcal{O}}(\epsilon^{-1}).$$
\end{theorem}

\begin{proof}
Let $m^{(l)}=\lceil3\Phi e^{l/2}\rceil $ and $K^{(l)}= \bar{K}$, then  we have $3\eta^{(l)}\Phi\le m^{(l)}\eta^{(l)}\le\frac{1}{6L\lceil1/r\rceil}$, and thus the proof process of the last Theorem implies that 
\begin{equation*}
F(\theta_{a_{l+1}})-F^* \le (1-\frac{m^{(l)}\eta^{(l)}\mu}{2})^{K^{(l)}}( F(\theta_{a_l})-F^*)+\frac{4C^2}{(m^{(l)})^2\mu}.
\end{equation*}
Denote $\kappa=\frac{\mu}{12L\lceil1/r\rceil}$ and take $l=0,1,\dots,J-1$, we obtain the following inequality:
\begin{equation*}
F(\theta_{a_{J}})-F^* \le (1-\kappa)^{\bar{K}J}( F(\theta_{0})-F^*)+\sum_{l=0}^{J-1}\frac{4C^2}{(m^{(l)})^2\mu}(1-\kappa)^{(J-1-l)\bar{K}}.
\end{equation*}
Let $\bar{K}=\lceil1/\kappa\rceil$, then $(1-\kappa)^{\bar{K}J}\le\exp(-\kappa\bar{K}J)\le\exp(-J)$, and we have
\begin{align*}
\sum_{l=0}^{J-1}\frac{4C^2}{(m^{(l)})^2\mu}(1-\kappa)^{(J-1-l)\bar{K}}\le\sum_{l=0}^{J-1}\frac{4C^2}{\lceil3\Phi\rceil^2e^l\mu}\exp(-\kappa\bar{K}(J-1-l)) \le\sum_{l=0}^{J-1}\frac{4C^2}{\lceil3\Phi\rceil^2\mu }e^{-J+1}= \frac{4C^2e}{\lceil3\Phi\rceil^2\mu}Je^{-J}.
\end{align*}
To make $F(\theta_{a_{J}})-F^*\le\epsilon^2$, it suffices to choose $J$ such that
\begin{equation*}
e^{-J}\Delta\le \frac{\epsilon^2}{2},\quad  \frac{4C^2e}{\lceil3\Phi\rceil^2\mu}Je^{-J}\le \frac{\epsilon^2}{2}.
\end{equation*}
Define $\nu = \max\{2\Delta, \frac{8C^2e}{\lceil3\Phi\rceil^2\mu}\}$, and take $J=\lceil\log(\nu/\epsilon^2)+\log\log(2\nu/\epsilon^2)\rceil$. Then we have $e^{-J}\Delta \le \frac{\epsilon^2}{2}$ and
\begin{align*}
\frac{4C^2e}{\lceil3\Phi\rceil^2\mu}Je^{-J}\le \frac{\log(\frac{\nu}{\epsilon^2})+\log\log(\frac{2\nu}{\epsilon^2})+1}{\log(\frac{2\nu}{\epsilon^2})} \frac{\epsilon^2}{2}\le \frac{\epsilon^2}{2},
\end{align*}
for sufficiently small $\epsilon>0$. Hence, the number of steps $T$ is bounded by
\begin{equation*}
T=\sum_{l=0}^{J-1}m^{(l)}\bar{K}\le\sum_{l=0}^{J-1}6\Phi\bar{K} e^{l/2}=6\Phi\bar{K}\frac{e^{J/2}-1}{e-1}\le \frac{6\sqrt{e}\Phi\bar{K}}{e-1}\left(\frac{\sqrt\nu}{\epsilon}\sqrt{\log\frac{2\nu}{\epsilon^2}}\right) =\tilde{\mathcal{O}}(\epsilon^{-1}).
\end{equation*}
This establishes the desired result.
\end{proof}

\subsection{Proof details of Section \ref{subsec:illustrative}}

Recall that we consider a linear regression problem:
$$ \min_{\theta}F(\theta)=\frac{1}{n}\sum_{i=1}^{n}f\big(\theta; x^{(i)}, y^{(i)}\big)
=\frac{1}{n}\sum_{i=1}^{n}\big((x^{(i)})^{\top}\theta - y^{(i)}\big)^{2}
=\frac{1}{2}\,\theta^{\top}A\theta - b^{\top}\theta + c,
$$
where the total sample set is $\{(x^{(i)}, y^{(i)})\}_{i=1}^{n}$, and
$$
A =\frac{2}{n}\sum_{i=1}^{n} x^{(i)}(x^{(i)})^{\top},
\quad
b = \frac{2}{n}\sum_{i=1}^{n} x^{(i)} y^{(i)},\quad c=\frac{1}{n}\sum_{i=1}^{n} (y^{(i)})^2.
$$
We optimize the parameters via stochastic gradient descent,
$$
\theta_{t+1} \;=\; \theta_t - \eta_t\, g_t,
$$
where $\eta_t$ denotes the learning rate and $g_t$ is a stochastic gradient. 

We first show that if the stochastic gradient $g_t$ is taken as:
\begin{itemize}
    \item\label{synthetic_iid_sample} \textbf{IID}: $g_t = \nabla f(\theta_t; x_t, y_t)$, where $(x_t, y_t)$ is sampled from $\{(x^{(i)}, y^{(i)})\}_{i=1}^{n}$ independently,
\item \label{synthetic_iid_sample_iid_mask} or \textbf{IID\_mask\_iid}: $g_t = S_t \odot \nabla f(\theta_t; x_t, y_t)$, where $S_t$ is an i.i.d.\ mask vector and $(x_t, y_t)$ is sampled from $\{(x^{(i)}, y^{(i)})\}_{i=1}^{n}$ independently,
\end{itemize}
then the convergence rate of $\rho_t$ will be lower bounded by $\Omega(t^{-1})$. We formulate this argument as the following theorem.

\begin{theorem}\label{thm:lower_bound_iid_sample_iid_mask}
If the stochastic gradient takes the form of IID or IID\_mask\_iid, $\{\eta_t\}$ are the learning rates such that $\rho_t$ is non-increasing and $0<\rho_t\le \mathcal {O}(t^{-1})$, then we have $\rho_t=\Omega(t^{-1})$. 
\end{theorem}

\subsubsection{Proof of Theorem \ref{thm:lower_bound_iid_sample_iid_mask}}

\begin{proof}
First we consider the form of \emph{IID}. From the SGD update, we have
\begin{equation*}
\theta_{t+1} -\theta^* = \theta_{t} -\theta^*-\eta_t g_t = (I-\eta_tA)(\theta_{t} -\theta^*) + \eta_t(\nabla F(\theta_t)-g_t).
\end{equation*}
Notice that $\mathbb{E}[g_t|\theta_t] = \nabla F(\theta_t)$, thus $\mathbb{E}\big[ \big\langle
(I-\eta_tA)(\theta_{t} -\theta^*),\eta_t(\nabla F(\theta_t)-g_t)\big\rangle\big]=0$ and 
\begin{align*}
    \mathbb{E}\|\theta_{t+1}-\theta^*\|^2 &= \mathbb{E}\|(I-\eta_tA)(\theta_{t}-\theta^*)\|^2+\eta_t^2\mathbb{E} \|\nabla F(\theta_{t})-g_t\|^2 \\
    & = \eta_t^2\left(\mathbb{E} \|\nabla F(\theta_{t})-g_t\|^2+\mathbb{E}\|A(\theta_{t}-\theta^*)\|^2\right) - 2\eta_t\mathbb{E}[(\theta_{t}-\theta^*)^\top A(\theta_{t}-\theta^*)] + \mathbb{E}\|\theta_{t}-\theta^*\|^2 \\
    &\ge \mathbb{E}\|\theta_{t}-\theta^*\|^2 -\frac{\left(\mathbb{E}[(\theta_{t}-\theta^*)^\top A(\theta_{t}-\theta^*)]\right)^2}{\mathbb{E} \|\nabla F(\theta_{t})-g_t\|^2+\mathbb{E}\|A(\theta_{t}-\theta^*)\|^2},
\end{align*}
where the last inequality comes from the minimum of the quadratic function. Furthermore, there exists large $t_0$ such that for $t\ge t_0$, we have
\begin{align*}
\mathbb{E} \|\nabla F(\theta_{t})-g_t\|^2&=\mathbb{E}\|(A-2x_tx_t^\top)(\theta_t-\theta^*)-2x_t(x_t^\top\theta^*-y_t)\|^2    \\
&\ge 2\mathbb{E} \|x_t(x_t^\top\theta^*-y_t)\|^2 - \mathbb{E}\|(A-2x_tx_t^\top)(\theta_t-\theta^*)\|^2 \\
&\ge 2\mathbb{E} \|x_t(x_t^\top\theta^*-y_t)\|^2-\mathcal {O}\left(\frac{1}{t}\right)\ge c,
\end{align*}
where $c:=\mathbb{E} \|x_t(x_t^\top\theta^*-y_t)\|^2=\frac{1}{n}\sum_{i=1}^{n}\left\| 
x^{(i)}({x^{(i)}})^\top\theta^*-x^{(i)}y^{(i)}\right\|^2>0$.

Hence, we have 
\begin{align*}
    &\rho_{t+1}\ge\rho_{t}-\frac{\lambda_{\max}^2\rho_t^2}{\lambda_{\min}^2\rho_t + c}=\rho_{t}\left(\frac{c-(\lambda_{\max}^2-\lambda_{\min}^2)\rho_t}{\lambda_{\min}^2\rho_t + c}\right).
\end{align*}
Divide both sides of the inequality by $\rho_t\rho_{t+1}$, we obtain
\begin{align*}
    \frac{1}{\rho_{t}}-\frac{1}{\rho_{t+1}}\ge-\frac{\lambda_{\max}^2\rho_t}{\rho_{t+1}(\lambda_{\min}^2\rho_t + c)} \ge -\frac{\lambda_{\max}^2}{c-(\lambda_{\max}^2-\lambda_{\min}^2)\rho_t}.
\end{align*}
Since $\rho_t\le \mathcal {O}(\frac{1}{t})$, there exists $t_1$ such that $c-(\lambda_{\max}^2-\lambda_{\min}^2)\rho_{t_1}>0$. It implies that for $t>\max\{t_0,t_1\}\triangleq t_2$, we have
\begin{align*}
    \frac{1}{\rho_t}\le\frac{1}{\rho_{t_2}}+\sum_{i=t_2}^{t-1}\frac{\lambda_{\max}^2}{c-(\lambda_{\max}^2-\lambda_{\min}^2)\rho_i}\le\frac{1}{\rho_{t_2}}+(t-1-t_2)\frac{\lambda_{\max}^2}{c-(\lambda_{\max}^2-\lambda_{\min}^2)\rho_{t_2}},
\end{align*}
which gives the desired result of $\rho_t = \Omega(\frac{1}{t})$. Analogously, we can show the same result if the stochastic gradient takes the form of \emph{IID\_mask\_iid}.
\end{proof}

% \begin{theorem}\label{thm:lower_bound_wor_sample_iid_mask_or_proj}
% If the stochastic gradient takes the form \ref{synthetic_wor_sample_iid_mask} or \ref{synthetic_wor_sample_iid_proj}, $\{\eta_t\}$ are the learning rates that satisfy $\frac{c_0}{t}\le\eta_t \le\frac{c_1}{t}$, $|\eta_t-\eta_{t+1}|\le \mathcal {O}(\frac{1}{t^2})$ for large enough $t$, $\alpha:=c_0\lambda_{\min}>\frac{1}{2}$,  and $0<\rho_t\le \mathcal {O}(\frac{1}{t})$, then we have $\rho_t=\Omega(\frac{1}{t})$. 
% \end{theorem}

\subsubsection{Proof of Theorem \ref{thm:lower_bound_wor_sample_iid_mask_or_proj}}

\begin{proof}
The SGD update gives
\begin{align*}
\theta_{t+1} -\theta^*  &= (I-\eta_tA)(\theta_{t} -\theta^*) + \eta_t(\nabla F(\theta_t)-g_t) \\
&= \left[\prod_{u=0}^t(I-\eta_uA)\right](\theta_0-\theta^*)+\sum_{u=0}^t \left[\prod_{i=u}^{t-1}(I-\eta_iA)\right]\eta_u(\nabla F(\theta_u)-g_u),
\end{align*}
where we denote $\prod_{i=t}^{t-1}(I-\eta_iA)\triangleq I$. Since $\left\|\prod_{u=0}^t(I-\eta_uA)\right\|_{2}\le \prod_{u=0}^t(1-\eta_u\lambda_{\min})\lesssim t^{-c_0\lambda_{\min}}$, we have
\begin{equation}\label{decay_term}
\left\|\left[\prod_{u=0}^t(I-\eta_uA)\right](\theta_0-\theta^*)\right\|^2 \lesssim t^{-2\alpha},
\end{equation}
here $\alpha:=c_0\lambda_{\min}$ and $\|B\|_2=\sup_{v\neq0}\frac{\|Bv\|}{\|v\|}$ is defined as the spectrum norm of a matrix. Next, we give the lower bound of the squared $L^2$ norm of the second term: 
\begin{align*}
&\ \left\|\sum_{u=0}^t \left[\prod_{i=u}^{t-1}(I-\eta_iA)\right]\eta_u(\nabla F(\theta_u)-g_u)\right\|^2 \\
=&\ \left\|\sum_{u=0}^t \left[\prod_{i=u}^{t-1}(I-\eta_iA)\right]\eta_u\big
(\nabla F(\theta_u)-\nabla f(\theta_u;x_u,y_u)+\nabla f(\theta_u;x_u,y_u)-g_u\big)\right\|^2 \\
\ge &\ \frac{1}{2}\underbrace{\left\|\sum_{u=0}^t \left[\prod_{i=u}^{t-1}(I-\eta_iA)\right]\eta_u
(g_u-\nabla  f(\theta_u;x_u,y_u))\right\|^2}_{{\displaystyle{I_1}}} -\underbrace{\left\|\sum_{u=0}^t \left[\prod_{i=u}^{t-1}(I-\eta_iA)\right]\eta_u
(\nabla F(\theta_u)-\nabla f(\theta_u;x_u,y_u))\right\|^2}_{{\displaystyle{I_2}}}.
\end{align*}
By Lemma \ref{lem:lower_bound_I_1} and \ref{lem:upper_bound_I_2}, we obtain $\mathbb{E}I_1\gtrsim\frac{1}{t}$ and $\mathbb{E}I_2\lesssim\frac{1}{t^{2\alpha}}(\log t)^2$. Combining with \eqref{decay_term}, we conclude that if $\alpha>\frac{1}{2}$, then
\begin{align*}
\mathbb{E}\|\theta_{t+1}-\theta^*\|^2 &\ge \frac{1}{2}\mathbb{E}\left\|\sum_{u=0}^t \left[\prod_{i=u}^{t-1}(I-\eta_iA)\right]\eta_u(\nabla F(\theta_u)-g_u)\right\|^2 -\mathbb{E}\left\|\left[\prod_{u=0}^t(I-\eta_uA)\right](\theta_0-\theta^*)\right\|^2
\\&\ge \frac{1}{4}\mathbb{E}I_1-\frac{1}{2}\mathbb{E}I_2-\mathbb{E}\left\|\left[\prod_{u=0}^t(I-\eta_uA)\right](\theta_0-\theta^*)\right\|^2 \gtrsim \frac{1}{t},
\end{align*}
which completes the proof. 
\end{proof}

\begin{lemma}\label{lem:lower_bound_I_1}
Under the settings of Theorem \ref{thm:lower_bound_wor_sample_iid_mask_or_proj}, we have
\begin{equation*}
\mathbb{E}I_1=\mathbb{E}\left\|\sum_{u=0}^t \left[\prod_{i=u}^{t-1}(I-\eta_iA)\right]\eta_u
    (g_u-\nabla  f(\theta_u;x_u,y_u))\right\|^2 \gtrsim \frac{1}{t}.
\end{equation*}
\end{lemma}

\begin{proof}
If $g_t$ takes the form of \emph{RR\_mask\_iid}, since mask vectors $S_u$ are i.i.d. and $\mathbb{E}S_u =\mathbf{1}_d$, we have 
\begin{align*}
    &\ \mathbb{E}\langle(S_i-\mathbf{1}_d)\odot\nabla  f(\theta_i;x_i,y_i),(S_j-\mathbf{1}_d)\odot\nabla  f(\theta_j;x_j,y_j)\rangle \\
    =&\ \mathbb{E}\Big[\mathbb{E}\big[\langle(S_i-\mathbf{1}_d)\odot\nabla  f(\theta_i;x_i,y_i),(S_j-\mathbf{1}_d)\odot\nabla  f(\theta_j;x_j,y_j)\rangle|\theta_i,x_i,y_i,S_i,\theta_j,x_j,y_j\big]\Big] \\
    =&\ \mathbb{E}\Big[\big\langle(S_i-\mathbf{1}_d)\odot\nabla  f(\theta_i;x_i,y_i),\mathbb{E}\big[(S_j-\mathbf{1}_d)\odot\nabla  f(\theta_j;x_j,y_j)|\theta_i,x_i,y_i,S_i,\theta_j,x_j,y_j\big]\big\rangle\Big] \\
    = &\ \mathbb{E}\big[\big\langle(S_i-\mathbf{1}_d)\odot\nabla  f(\theta_i;x_i,y_i), 0\big\rangle\big] = 0, \quad \text{$\forall$ $i<j$.}
\end{align*}
Then we can lower bound $\mathbb{E}I_1$ as follows:
\begin{align*}
\mathbb{E}I_1 &= \sum_{u=0}^t\mathbb{E}\left\| \left[\prod_{i=u}^{t-1}(I-\eta_iA)\right]\eta_u
(S_u-\mathbf{1}_d)\odot\nabla  f(\theta_u;x_u,y_u)\right\|^2 \\
&\ge \sum_{u=0}^t\left[\prod_{i=u}^{t-1}(1-\eta_i\lambda_{\max})\right]^2\eta_u^2\mathbb{E}\left\| 
(S_u-\mathbf{1}_d)\odot\nabla  f(\theta_u;x_u,y_u)\right\|^2 \\
& \gtrsim \sum_{u=\lceil \frac{t}{2}\rceil}^t\left[\prod_{i=u}^{t-1}\left(1-\frac{c_1\lambda_{\max}}{i}\right)\right]^2\frac{c_0^2}{u^2}\mathbb{E}\left\| 
(S_u-\mathbf{1}_d)\odot\nabla  f(\theta_u;x_u,y_u)\right\|^2.
\end{align*}
Also, we note that for large $u$, we have
\begin{align*}
\mathbb{E}\big\| 
(S_u-\mathbf{1}_d)&\odot\nabla  f(\theta_u;x_u,y_u)\big\|^2 = \frac{1-r}{r}\mathbb{E}\left\| 
\nabla  f(\theta_u;x_u,y_u)\right\|^2 \\
&= \frac{1-r}{r}\mathbb{E}\left\|2 
x_ux_u^\top(\theta_u-\theta^*) +2x_u(x_u^\top\theta^*-y_u)\right\|^2
\ge \frac{1-r}{2r}\mathbb{E}\left\|2 
x_u(x_u^\top\theta^*-y_u)\right\|^2-\mathcal {O}\left(\frac{1}{u}\right) \ge \sigma^2_0,
\end{align*}
where $\sigma_0^2:=\frac{1-r}{r}\mathbb{E}\left\| 
x_u(x_u^\top\theta^*-y_u)\right\|^2=\frac{1-r}{rn}\sum_{i=1}^{n}\left\| 
x^{(i)}({x^{(i)}})^\top\theta^*-x^{(i)}y^{(i)}\right\|^2>0$.

Thus, for sufficiently large $t$, we obtain the following inequality:
\begin{equation}\label{I_1}
\begin{aligned}
  \mathbb{E}I_1&\gtrsim \sum_{u=\lceil \frac{t}{2}\rceil}^t\left[\prod_{i=u}^{t-1}\left(1-\frac{c_1\lambda_{\max}}{i}\right)\right]^2\frac{c_0^2\sigma_0^2}{u^2} 
\ge \sum_{u=\lceil \frac{t}{2}\rceil}^t\left(1-\frac{2c_1\lambda_{\max}}{t}\right)^{\frac{t}{2}}\frac{c_0^2\sigma_0^2}{u^2} \\&\gtrsim \sum_{u=\lceil \frac{t}{2}\rceil}^te^{-2c_1\lambda_{\max}}\frac{c_0^2\sigma_0^2}{u^2}\ge \frac{e^{-2c_1\lambda_{\max}}c_0^2\sigma_0^2}{2t}.  
\end{aligned}
\end{equation}
Analogously, if $g_t$ takes the form of \emph{RR\_proj}, then 
\begin{align*}
    & \mathbb{E}\left\langle\left(\frac{1}{r}P_iP_i^T-I\right)\nabla  f(\theta_i;x_i,y_i),\left(\frac{1}{r}P_jP_j^\top-I\right)\nabla  f(\theta_j;x_j,y_j)\right\rangle = 0, \quad \text{$\forall$ $i<j$.}
\end{align*}
Hence we have
\begin{align*}
\mathbb{E}I_1 &= \sum_{u=0}^t\mathbb{E}\left\| \left[\prod_{i=u}^{t-1}(I-\eta_iA)\right]\eta_u
\left(\frac{1}{r}P_uP_u^\top-I\right)\nabla  f(\theta_u;x_u,y_u)\right\|^2 \\
&\ge \sum_{u=0}^t\left[\prod_{i=u}^{t-1}(1-\eta_i\lambda_{\max})\right]^2\eta_u^2\mathbb{E}\left\| 
\left(\frac{1}{r}P_uP_u^\top-I\right)\nabla  f(\theta_u;x_u,y_u)\right\|^2 \\
& \gtrsim \sum_{u=\lceil \frac{t}{2}\rceil}^t\left[\prod_{i=u}^{t-1}\left(1-\frac{c_1\lambda_{\max}}{i}\right)\right]^2\frac{c_0^2}{u^2}\mathbb{E}\left\| 
\left(\frac{1}{r}P_uP_u^\top-I\right)\nabla  f(\theta_u;x_u,y_u)\right\|^2.
\end{align*}
Note that $\frac{1}{r}P_uP_u^\top$ serves as low-rank projection and thus
\begin{align*}
\mathbb{E}\bigg\| 
\bigg(\frac{1}{r}P_uP_u^\top-I\bigg)\nabla  f(\theta_u;x_u,y_u)\bigg\|^2&=\mathbb{E}\left[(\nabla  f(\theta_u;x_u,y_u))^\top 
\left(\frac{1}{r^2}P_uP_u^\top-\frac{2}{r}P_uP_u^\top+I\right)\nabla f(\theta_u;x_u,y_u)\right] \\ &= \frac{1-r}{r}\mathbb{E}\left\| 
\nabla  f(\theta_u;x_u,y_u)\right\|^2 \ge \sigma_0^2.
\end{align*}
Then \eqref{I_1} implies that $\mathbb{E}I_1\gtrsim\frac{1}{t}$.
\end{proof}

\begin{lemma}\label{lem:upper_bound_I_2}
Under the settings of Theorem \ref{thm:lower_bound_wor_sample_iid_mask_or_proj}, we have
\begin{equation*}
\mathbb{E}I_2=\mathbb{E}\left\|\sum_{u=0}^t \left[\prod_{i=u}^{t-1}(I-\eta_iA)\right]\eta_u
(\nabla F(\theta_u)-\nabla f(\theta_u;x_u,y_u))\right\|^2 \lesssim\;
\begin{cases}
\ t^{-2\alpha}, & 0<\alpha<1,\\
\ t^{-2}(\log t)^2, & \alpha=1,\\
\ t^{-2}, & \alpha>1.
\end{cases}
\end{equation*}
\end{lemma}

\begin{proof}
For simplicity, we introduce some notation before deriving the upper bound for $\mathbb{E}I_2$. Denote $\Psi_{t,u}\coloneqq\prod_{i=u}^{t-1}\big(I-\eta_i A\big)$, $ \ \xi_u:=\nabla f(\theta_u;x_u,y_u)-\nabla F(\theta_u)$. Let $t=Kn+m\ (K=\lfloor\frac{t}{n}\rfloor, m\in\{0,\dots,n-1\})$. For $u\le t$, write $u=kn+s$, where $k\in\{0,\dots,K\}, s\in\{0,\dots,n-1\}$. Denote $i_{k,s}\in\{1,\dots,n\}\text{\ as the sample index of random reshuffle at step $u$}$. Moreover, let $\theta_{k,s}:=\theta_{kn+s},W_{k,s}:=\Psi_{t,kn+s}\,\eta_{kn+s}, \ \xi_{k,s}:=\nabla f(\theta_{k,s};x^{(i_{k,s})},y^{(i_{k,s})})-\nabla F(\theta_{k,s})$, then 
\begin{equation*}
I_2=\bigg\|\sum_{u=0}^t\Psi_{t,u}\eta_u\xi_u\bigg\|^2=\bigg\|\sum_{k=0}^{K-1} \sum_{s=0}^{n-1}W_{k,s}\xi_{k,s}+\sum_{s=0}^{m}W_{K,s}\xi_{K,s}\bigg\|^2.   
\end{equation*}

We divide the proof of the upper bound for $\mathbb{E}I_2$ into the following steps. 

\paragraph{Step 1: Rate of the decay term.}
By spectral calculus, we obtain
\begin{align*}
\|\Psi_{t,u}\|_2\lesssim\prod_{i=u}^{t-1}\left|1-\frac{c_0\lambda_{\min}}{i}\right|\lesssim\; \Big(\frac{u+1}{t+1}\Big)^{\alpha},\quad \text{where\ } \alpha=c_0\lambda_{\min}.
\end{align*}

\paragraph{Step 2: Decomposing $\xi_{k,s}$.}
Write
\begin{align*}
\xi_{k,s}
&=(A_{i_{k,s}}-A)\,\theta_{k,s}-(b_{i_{k,s}}-b)\\
&=(A_{i_{k,s}}-A)\,(\theta_{k,s}-\theta_{k,0}) +\underbrace{{\big[(A_{i_{k,s}}-A)\,\theta_{k,0}-(b_{i_{k,s}}-b)\big]}}_{\theta^\circ_{k,s}} , 
\end{align*}
where $A_{i_{k,s}}:= 2x^{(i_{k,s})}(x^{(i_{k,s})})^\top,\, b_{i_{k,s}}:= 2x^{(i_{k,s})}y^{(i_{k,s})}$.
Therefore,
\begin{equation}
\label{eq:step2_total}
  \sum_{s=0}^{n-1}W_{k,s}\xi_{k,s}
=\sum_{s=0}^{n-1}W_{k,s}(A_{i_{k,s}}-A)(\theta_{k,s}-\theta_{k,0})+\sum_{s=0}^{n-1}W_{k,s}\theta^\circ_{k,s}
.  
\end{equation}

% Notice that $\theta_{k,s}-\theta_{k,0}=\sum_{j=0}^{s-1}\eta_{kn+j}\,g_{kn+j},$ by Minkowski's inequality, we have
% \begin{align*}
% \Big(\mathbb{E}\max_s\|\theta_{k,s}-\theta_{k,0}\|^2\Big)^{1/2}\le\sum_{j=0}^{n-1}(\mathbb{E}\|\eta_{kn+j}\,g_{kn+j}\|^2)^{1/2}\lesssim\;\sum_{j=0}^{n-1}\eta_{kn+j}\;\lesssim\;\frac{1}{k}.
% \end{align*}
Notice that $\theta_{k,s}-\theta_{k,0}=(\theta_{k,s}-\theta^*)-(\theta_{k,0}-\theta^*),$ by Minkowski's inequality, we have
\begin{align*}
\Big(\mathbb{E}\|\theta_{k,s}-\theta_{k,0}\|^2\Big)^{1/2}\le\Big(\mathbb{E}\|\theta_{k,s}-\theta^*\|^2\Big)^{1/2} + \Big(\mathbb{E}\|\theta_{k,0}-\theta^*\|^2\Big)^{1/2}\lesssim\;\frac{1}{k}.
\end{align*}
Together with $\sum_{s=0}^{n-1}\|W_{k,s}\|_2\le(\max_{0\le s\le n-1}\|\Psi_{t,kn+s}\|_2)\sum_{j=0}^{n-1}\eta_{kn+j}\lesssim (k/K)^\alpha\cdot (1/k)$, we obtain

\begin{equation}
\label{eq:step2_first_term}
 \Big(\mathbb{E}\Big\|\sum_{s}W_{k,s}(A_{i_{k,s}}-A)(\theta_{k,s}-\theta_{k,0})\Big\|^2\Big)^{1/2}
\;\lesssim\;\Big(\frac{k}{K}\Big)^{\alpha}\cdot \frac{1}{k^2}.   
\end{equation}

\paragraph{Step 3: Epochwise Abel transform.}
Let  $T_{k,s}:=\sum_{j=0}^{s}\theta^\circ_{k,j}.$ Random reshuffle implies that $T_{k,n-1}=\sum_{s=0}^{n-1}\theta^\circ_{k,s}=0$ (each sample appears once within one epoch).
Therefore the Abel transform gives the equation
\begin{align*}
\sum_{s=0}^{n-1}W_{k,s}\,\theta^\circ_{k,s}
\;=\;\sum_{s=0}^{n-2}\big(W_{k,s}-W_{k,s+1}\big)\,T_{k,s}.
\end{align*}
Taking $L^2$ norm of expectation and using Minkowski's inequality,
\begin{align*}
\Big(\mathbb{E}\Big\|\sum_{s=0}^{n-1}W_{k,s}\,\theta^\circ_{k,s}\Big\|^2\Big)^{1/2}
\;\le\;\sum_{s=0}^{n-2}\|W_{k,s}-W_{k,s+1}\|_2
\,\Big(\mathbb{E}\|T_{k,s}\|^2\Big)^{1/2}
.
\end{align*}
Since $\sup_{k,s}\mathbb{E}\|\theta^\circ_{k,s}\|^2<\infty$, we have
\begin{align*}
\Big(\mathbb{E}\|T_{k,s}\|^2\Big)^{1/2}\le \sum_{j=0}^{s}\Big(\mathbb{E}\|\theta^\circ_{k,j}\|^2\Big)^{1/2}\lesssim\; 1.
\end{align*}
Using $\Psi_{t,u}-\Psi_{t,u+1}=-\eta_u A\,\Psi_{t,u+1}$ and $\eta_u-\eta_{u+1}\le \mathcal {O}(1/u^2)$, we have
\begin{align*}
\|W_{k,s}-W_{k,s+1}\|_2
&=\|\Psi_{t,kn+s}\eta_{kn+s}-\Psi_{t,kn+s+1}\eta_{kn+s+1}\|_2\\
&\le \|\Psi_{t,kn+s+1}\|_2\big(\|A\|\eta_{kn+s}^2+|\eta_{kn+s}-\eta_{kn+s+1}|\big)
\lesssim \Big(\frac{kn+1}{t+1}\Big)^{\alpha}\cdot \frac{1}{(kn)^2}.
\end{align*}
Consequently, we can derive that
\begin{equation}
\label{eq:step2_second_term}
 \Big(\mathbb{E}\Big\|\sum_{s=0}^{n-1}W_{k,s}\,\theta^\circ_{k,s}\Big\|^2\Big)^{1/2}\lesssim\sum_{s=0}^{n-2}\|W_{k,s}-W_{k,s+1}\|_2
\;\lesssim\; \Big(\frac{k}{K}\Big)^{\alpha}\cdot \frac{1}{k^2}.   
\end{equation}

% Combining  with \eqref{eq:step2_total} and \eqref{eq:step2_first_term} yields the epochwise $L^2$ bound that
Therefore, combining \eqref{eq:step2_second_term} with \eqref{eq:step2_total} and \eqref{eq:step2_first_term} leads to the desired epochwise $L^2$ bound:
\begin{align*}
\Big(\mathbb{E}\Big\|\sum_{s=0}^{n-1}W_{k,s}\xi_{k,s}\Big\|^2\Big)^{1/2}
\;\lesssim\;\Big(\frac{k}{K}\Big)^{\alpha}\cdot \frac{1}{k^2}.
\end{align*}

\paragraph{Step 4: Aggregating epochs.}
Let $Z_k\triangleq \sum_{s=0}^{n-1}W_{k,s}\xi_{k,s}$. By Minkowski's inequality, we obtain
\begin{align*}
\Big(\mathbb{E}\Big\|\sum_{k=1}^{K-1} Z_k\Big\|^2\Big)^{1/2}
\;\le\;\sum_{k=1}^{K-1}\big(\mathbb{E}\|Z_k\|^2\big)^{1/2}
\;\lesssim\;\sum_{k=1}^{K-1}\Big(\frac{k}{K}\Big)^{\alpha}\frac{1}{k^2}
\;=\;K^{-\alpha}\sum_{k=1}^{K-1}k^{\alpha-2}.
\end{align*}

Note that 
\begin{align*}
\Big(\mathbb{E}\Big\|\sum_{u=0}^{t}\Psi_{t,u}\eta_u\xi_u\Big\|^2
\Big)^{1/2}&\le\Big(\mathbb{E}\Big\|\sum_{k=1}^{K-1} Z_k\Big\|^2\Big)^{1/2}+\Big(\mathbb{E}\Big\|\sum_{s=0}^{m}W_{K,s}\xi_{K,s}\Big\|^2\Big)^{1/2} +\Big(\mathbb{E}\Big\|\sum_{s=0}^{n-1}W_{0,s}\xi_{0,s}\Big\|^2\Big)^{1/2}
\\&\lesssim K^{-\alpha}\sum_{k=1}^{K-1}k^{\alpha-2} + K^{-1} + K^{-\alpha},
\end{align*}
which implies that
\begin{align*}
\mathbb{E}\Big\|\sum_{u=0}^{t}\Psi_{t,u}\eta_u\xi_u\Big\|^2
\lesssim\;
\begin{cases}
\ K^{-2\alpha}, & 0<\alpha<1,\\
\ K^{-2}(\log K)^2, & \alpha=1,\\
\ K^{-2}, & \alpha>1,
\end{cases}
\end{align*}
because $\sum_{k=1}^{K-1}k^{\alpha-2}\lesssim 1$ for $\alpha<1$, $\lesssim \log K$ for $\alpha=1$, and $\lesssim K^{\alpha-1}$ for $\alpha>1$. Since $K\asymp t/n$, we obtain
\begin{align}\label{I_2}
\mathbb{E}I_2=\mathbb{E}\Big\|\sum_{u=0}^{t}\Psi_{t,u}\eta_u\xi_u\Big\|^2 \;\lesssim\;
\begin{cases}
\ t^{-2\alpha}, & 0<\alpha<1,\\
\ t^{-2}(\log t)^2, & \alpha=1,\\
\ t^{-2}, & \alpha>1.
\end{cases}
\end{align}
This establishes the desired inequality.
\end{proof}

% \begin{theorem}\label{thm:upper_bound_wor_sample_wor_mask}
% If the stochastic gradient takes the form \ref{synthetic_wor_sample} or \ref{synthetic_wor_sample_wor_mask}, $\{\eta_t\}$ are the learning rates that satisfy $\frac{c_0}{t}\le\eta_t \le\frac{c_1}{t}$, $|\eta_t-\eta_{t+1}|\le \mathcal {O}(\frac{1}{t^2})$ for large enough $t$, $\alpha:=c_0\lambda_{\min}>2$, then we have $\rho_t\le \mathcal {O}(\frac{1}{t^2})$. 
% \end{theorem}
\subsubsection{Proof of Theorem \ref{thm:upper_bound_wor_sample_wor_mask}}

\begin{proof}
The proof can be divided into three parts. First, we prove that there exists a subsequence $\{t_k\}$ such that $F(\theta_{t_k})-F^*\le\mathcal{O}(\frac{1}{t_k})$. Second, we show that $\rho_t\le\mathcal{O}(t^{-1})$ . Finally, we claim that $\rho_t\le\mathcal{O}(t^{-2})$ by leveraging Lemma \ref{lem:upper_bound_I_2}.

\paragraph{Step 1: $F(\theta_{t_k})-F^*\le\mathcal{O}(\frac{1}{t_k})$.}
We first find a subsequence $\{t_{k}\}$ such that $\sup_{k}\|\theta_{t_k}\|<\infty$. By Lemma \ref{descent lem}, if we can choose $\{t_k\}$ such that $3\eta_{t_k}\Phi\le\sum_{i=t_k}^{t_{k+1}-1}\eta_i\le\frac{1}{6L\lceil1/r\rceil}$, then we obtain
\begin{equation}\label{eq: descent_lem_subsequence}
F(\theta_{t_{k+1}})\le F(\theta_{t_k})-\frac{\sum_{i=t_k}^{t_{k+1}-1}\eta_i}{4}\|\nabla F(\theta_{t_k})\|^2 + \frac{2}{\sum_{i=t_k}^{t_{k+1}-1}\eta_i}\eta_{t_k}^2C^2.
\end{equation}
Take $t_0$ sufficient large such that $\frac{c_1\max\{3\Phi,1\}}{t_0}<\frac{1}{12L\lceil1/r\rceil}$ and define $t_k$ sequentially as $t_{k+1}:=\inf\{t>t_k: \sum_{i=t_k}^{t-1}\eta_i\ge\frac{1}{12L\lceil1/r\rceil}\}$. Since $\eta_i\ge\frac{c_0}{i}$ and $\sum_{i\ge 1}\frac{1}{i}=\infty$, we know that $t_{k}<\infty$ is well defined and $\lim_kt_k=\infty$. By \eqref{eq: descent_lem_subsequence}, we obtain
\begin{equation}\label{eq: descent_lem_subsequence_2}
\begin{aligned}
F(\theta_{t_{k+1}})-F^*&\le F(\theta_{t_k})-F^*-\frac{1}{48L\lceil1/r\rceil}\|\nabla F(\theta_{t_k})\|^2 + 24C^2L\lceil1/r\rceil\eta_{t_k}^2 \\
&\le \left(1-\frac{\mu}{24L\lceil1/r\rceil}\right)(F(\theta_{t_k})-F^*) + \frac{24c_1^2C^2L\lceil1/r\rceil}{t_k^2}.
\end{aligned}
\end{equation}

For any $t$ such that $\sum_{i=t_k}^{t-1}\eta_i\ge\frac{1}{12L\lceil1/r\rceil}$, it is sufficient to let $\frac{t}{t_k}\ge e^{\frac{1}{12c_0L\lceil1/r\rceil}}$ since $$\sum_{i=t_k}^{t-1}\eta_i\ge\sum_{i=t_k}^{t-1}\frac{c_0}{i}\ge c_0\int_{t_k}^t\frac{1}{u}du=c_0\log\frac{t}{t_k}.$$ Hence by the construction of $t_{k+1}$, we conclude that $t_{k+1}\le \lceil e^{\frac{1}{12c_0L\lceil1/r\rceil}}t_k\rceil\le e^{\frac{1}{12c_0L\lceil1/r\rceil}}t_k+1$, i.e., $$\frac{t_{k+1}}{t_{k}}\le e^{\frac{1}{12c_0L\lceil1/r\rceil}}+\frac{1}{t_k}\le1+\frac{1}{12c_0L\lceil1/r\rceil}+\frac{1}{t_k}.$$
Note that $\mu=\lambda_{\min}$ and $\alpha=c_0\lambda_{\min}>2$, then we can choose sufficient large $K_0$ such that for all $k\ge K_0$ there holds 
\begin{align*}
&\ \left(1-\frac{\mu}{24L\lceil1/r\rceil}+\frac{24c_1^2C^2L\lceil1/r\rceil}{t_k}\right)\left(1+\frac{1}{12c_0L\lceil1/r\rceil}+\frac{1}{t_k}\right)\\\le&\ \left(1-\frac{\mu}{24L\lceil1/r\rceil}+\frac{24c_1^2C^2L\lceil1/r\rceil}{t_k}\right)\left(1+\frac{\mu}{24L\lceil1/r\rceil}+\frac{1}{t_k}\right) = 1-\left(\frac{\mu}{24L\lceil1/r\rceil}\right)^2 + \mathcal{O}\left(\frac{1}{t_k}\right) \le 1.
\end{align*}
It is obvious that there exists $\tilde{C}\ge1$ such that for $k<K_0$, $F(\theta_{t_k})-F^*\le\frac{\tilde{C}}{t_k}$. We argue that the above result also holds for $k\ge K_0$. We prove by induction. Suppose that we already have shown the result for $k$, then by \eqref{eq: descent_lem_subsequence_2}:
\begin{align*}
F(\theta_{t_{k+1}})-F^* &\le \left(1-\frac{\mu}{24L\lceil1/r\rceil}\right)(F(\theta_{t_k})-F^*) + \frac{24c_1^2C^2L\lceil1/r\rceil}{t_k^2} \\
&\le \left(1-\frac{\mu}{24L\lceil1/r\rceil}+\frac{24c_1^2C^2L\lceil1/r\rceil}{\tilde{C}t_k}\right)\frac{\tilde{C}}{t_k} \le \frac{\tilde{C}}{\left(1+\frac{1}{12c_0L\lceil1/r\rceil}+\frac{1}{t_k}\right)t_k}\le\frac{\tilde{C}}{t_{k+1}},
\end{align*}
therefore we conclude that for all $k$, $F(\theta_{t_{k}})-F^*\le\frac{\tilde{C}}{t_k}$.

\paragraph{Step 2: $\rho_t\le \mathcal{O}(t^{-1})$.} 
We use the notation from Step 1. There exists $\bar{C}$ such that for $t<t_{K_0+1}$, we have $F(\theta_t)-F^*\le\frac{\bar{C}}{t}$. We prove the desired result by induction again. Let $k\ge K_0+1$. Suppose it already holds for $t\le t_k$, we consider the case where $t_k< t\le t_{k+1}$. If $t\in(t_k,t_{k+1})$, then by the definition of $t_{k+1}$, we have $\sum_{i=t_k}^{t-1}\eta_i<\frac{1}{12L\lceil1/r\rceil}$. Define $t':=\sup\{t'<t_k:\sum_{i=t'}^{t-1}\eta_i\ge\frac{1}{12L\lceil1/r\rceil}\}$, then $t'\ge te^{-\frac{1}{12c_0L\lceil1/r\rceil}}-1$, $t'\in [t_{k-1},t_k)$ and $\sum_{i=t'}^{t-1}\eta_i\le\frac{1}{6L\lceil1/r\rceil}$. Using Lemma \ref{descent lem}, we obtain
\begin{align*}
F(\theta_t)-F^*&\le \left(1-\frac{\mu}{24L\lceil1/r\rceil}\right)(F(\theta_{t'})-F^*) + \frac{24c_1^2C^2L\lceil1/r\rceil}{(t')^2}. \\
&\le \left(1-\frac{\mu}{24L\lceil1/r\rceil}\right)\frac{\bar{C}}{t'} + \frac{24c_1^2C^2L\lceil1/r\rceil}{(t')^2} \\
&\le \frac{\bar{C}}{te^{-\frac{1}{12c_0L\lceil1/r\rceil}}-1} + \frac{24c_1^2C^2L\lceil1/r\rceil}{(te^{-\frac{1}{12c_0L\lceil1/r\rceil}}-1)^2}\lesssim \frac{1}{t}.
\end{align*}
For $t=t_{k+1}$, Step 1 implies that $F(\theta_{t_{k+1}})-F^*\le\frac{\tilde{C}}{t_{k+1}}$. Overall, we prove that $F(\theta_t)-F^*\le\frac{\max\{\bar{C},\tilde{C}\}}{t}$. It can be further inferred that $\rho_t\le \mathcal{O}(t^{-1})$ combined with $F(\theta_t)-F^*\ge\lambda_{\min}\|\theta_t-\theta^*\|^2.$ 

\paragraph{Step 3: $\rho_t\le \mathcal{O}(t^{-2})$.}
By the rule of SGD updates, we have
\begin{equation}
\label{eq:step3_sgd_update}
\begin{aligned}
\theta_{t+1} -\theta^*  &= (I-\eta_tA)(\theta_{t} -\theta^*) + \eta_t(\nabla F(\theta_t)-g_t) \\
&= \left[\prod_{u=0}^t(I-\eta_uA)\right](\theta_0-\theta^*)+\sum_{u=0}^t \left[\prod_{i=u}^{t-1}(I-\eta_iA)\right]\eta_u(\nabla F(\theta_u)-g_u).
\end{aligned}
\end{equation}
Since $\left\|\prod_{u=0}^t(I-\eta_uA)\right\|_{2}\le \prod_{u=1}^t(1-\eta_u\lambda_{\min})\lesssim t^{-c_0\lambda_{\min}}$, we have
\begin{equation}
\label{eq:step3_decay_term}
\left\|\left[\prod_{u=0}^t(I-\eta_uA)\right](\theta_0-\theta^*)\right\|^2 \lesssim t^{-2\alpha}.
\end{equation}
If $g_t$ take the form of \emph{RR}, then Lemma \ref{lem:upper_bound_I_2} implies that  
\begin{equation}
\label{eq:step3_cycle_wise_bound}
\mathbb{E}\left\|\sum_{u=0}^t \left[\prod_{i=u}^{t-1}(I-\eta_iA)\right]\eta_u(\nabla f(\theta_u)-g_u)\right\|^2\lesssim\frac{1}{t^{2}}.
\end{equation}
By closely revisiting the proof of Lemma~\ref{lem:upper_bound_I_2}, we can analogously show that the same result holds when \(g_t\) takes the form of \emph{RR\_mask\_wor}. The key modification is to partition the iterations into cycles—each of length \(\lceil 1/r\rceil\,n\), rather than using the epoch-wise partition adopted in the original lemma; the argument then proceeds cycle by cycle. 

By combining the bounds in \eqref{eq:step3_sgd_update}, \eqref{eq:step3_decay_term}, and \eqref{eq:step3_cycle_wise_bound}, we complete the proof.
\end{proof}

% in which $S_t$ is an i.i.d.\ mask vector generated at step $t$: each coordinate equals $1/r$ with probability $r$ and $0$ with probability $1-r$. At step $t$, the pair $(x_t, y_t)$ is sampled from $\{(x^{(i)}, y^{(i)})\}_{i=1}^{n}$ under the random-reshuffle protocol; that is, the sample sequence will be generated every $n$ steps, such that the entire dataset is traversed exactly once.

\section{Experimental Details}

\textbf{Platform}: 
All experiments are conducted on a single machine equipped with four NVIDIA RTX PRO 6000 GPUs.
%All the experiments are conducted on a machine with 4 NVIDIA RTX PRO 6000 Blackwell GPUs.

\subsection{Synthetic  Experiments}\label{appendix_synthetic_experiments}

We consider a linear regression problem with $n=1000$ fixed samples of dimension $d=10$. Features are generated as $x^{(i)} \sim \mathcal{N}(0, I_d)$, and responses as $y^{(i)} \mid x^{(i)} \sim \mathcal{N}((x^{(i)})^\top w_{\text{gen}}, 1)$, where $w_{\text{gen}} \in \mathbb{R}^d$ is drawn from $\text{Uniform}([0,1]^d)$. Let $\theta^* = A^{-1}b$ be the optimal solution to the least-squares problem, where $A = \frac{2}{n}\sum_{i=1}^n x^{(i)}(x^{(i)})^\top$ and $b = \frac{2}{n}\sum_{i=1}^n x^{(i)} y^{(i)}$. The optimizer runs for $T = 10^6$ iterations from initial point $\theta_0 = \mathbf{0}_d$ with diminishing step size $\eta_t$. We test gradient compression with mask ratio $r=0.5$ ($5$ coordinates are updated at each step). All compression methods (RR\_mask\_iid, RR\_mask\_wor, and RR\_proj) are activated after a warm-up of $100$ iterations. 

\subsection{Image Classification Tasks}

\textbf{Training ResNet-20 on CIFAR-10 and CIFAR-100.}  We train the model using SGD with Nesterov momentum of 0.9 and weight decay of $1\times10^{-4}$. 
The data is loaded with {random reshuffling} at the beginning of every epoch via a \verb|DataLoader| with \verb|shuffle=True|. 
We train the model for 200 epochs with a batch size of {256} and an initial learning rate of {0.1}.

The learning rate schedule differs between datasets:
\begin{itemize}
    \item For {CIFAR-10}, the learning rate is multiplied by 0.1 (i.e., \verb|gamma=0.1|) at epochs 100 and 150.
    \item For {CIFAR-100}, the learning rate is multiplied by 0.2 (i.e., \verb|gamma=0.2|) at epochs 60, 120, and 150.
\end{itemize}

\textbf{Training ResNet-18 on ImageNet.} The optimizer is SGDM with momentum of 0.9 and weight decay of $1\times10^{-4}$. The data is loaded with {random reshuffling} via \verb|shuffle=True|. 
We train the model for 100 epochs on {4 GPUs} with a batch size of {128} per GPU and an initial learning rate of {0.1}.  A multi-step learning rate schedule is used, where the learning rate is multiplied by 0.1 (i.e., \verb|gamma=0.1|) at epochs 30, 60, and 90.

\textbf{Fine-tuning Vit-base.} We fine-tune a pre-trained Vision Transformer model, which consists of an embedding layer, 12 Transformer encoder layers, and a final classification head.  Data is loaded with {random reshuffling} via a \verb|DataLoader| with \verb|shuffle=True|. We train the model on CIFAR-10/100 for 100 epochs and on ImageNet-1K for 10 epochs on 4 GPUs, using a batch size of 128 per GPU and a learning rate of $1\times10^{-4}$. The optimizer is AdamW with weight decay of  $1\times10^{-4}$. The learning rate was adjusted using a StepLR scheduler, which multiplied the learning rate by a factor of 0.95 every 2 epochs. For \textsc{LISA} and \textsc{LISA-wor}, we sample  $\gamma=3$ layers every $K=5$ epochs on CIFAR-10/100, and $\gamma=4$ layers every $K=1$ epoch on ImageNet-1K.

\subsection{Fine-tuning Experiments of RoBERTa}
We fine-tune pre-trained RoBERTa-Base model on the GLUE benchmark for 30 epochs on a single NVIDIA RTX PRO 6000. Training details including batch size, learning rate, maximum sequence length are illustrated in Table \ref{tab:glue-hyperparams}.

\begin{table}[h]
    \centering
    % \vspace{1em}
    \caption{Hyperparameters used in fine-tuning pre-trained RoBERTa-Base model on the GLUE benchmark.}
    \label{tab:glue-hyperparams}
    % \vspace{1em}
    \begin{tabular}{l|cccccccc}
        \toprule
        \textbf{\small Hyperparameter} & \textbf{\small CoLA} & \textbf{\small STS-B} & \textbf{\small MRPC} & \textbf{\small RTE} & \textbf{\small SST2} & \textbf{\small MNLI} & \textbf{\small QNLI} & \textbf{\small QQP}\\
        \midrule
        Batch Size & 32 & 16 & 16 & 32 & 16 & 16 & 16 & 64\\
        % \# Epochs & 30 & 30 & 30 & 30 & 30 & 30 & 30 & 30\\
        % Learning Rate & 2.5e-5 & 2.0e-5 & 3.5e-5 & 1.0e-5 & 1.0e-5 & 1.0e-5 & 1.0e-5 & 1.0e-5\\
        Learning Rate ($\times 10^{-5}$) & 2.5 & 2.0 & 3.5 & 1.0 & 1.0 & 1.0 & 1.0 & 1.0 \\
        Max Seq. Len. & 512 & 512 & 512 & 512 & 512 & 512 & 512 & 512\\
        \bottomrule
    \end{tabular}
    \vspace{1em}
\end{table}

\subsection{Pre-training Experiments of LLMs}

\textbf{GPT-2 pre-training.} We use the nanoGPT codebase\footnote{https://github.com/karpathy/nanoGPT/tree/master} to train GPT-2 sized 124M on OpenWebtext. The model architecture follows the standard GPT-2 configuration with 12 layers, 12 attention heads, and an embedding dimension of 768. Training is performed for 100,000 iterations with a batch size of 60 per GPU, a block size (context length) of 1024, and gradient accumulation over 8 steps, resulting in approximately 0.5M tokens per iteration. We use the AdamW optimizer with \(\beta_1 = 0.9\), \(\beta_2 = 0.95\), a peak learning rate of $6\times10^{-4}$, and a weight decay of 0.1. The learning rate schedule consists of a linear warmup over the first 2000 iterations, followed by a cosine decay to a minimum value of $6\times10^{-5}$.

% \textbf{GPU memory consumption.} To evaluate GPU memory cost of \textsc{LISA-wor}, we conduct experiments using a single NVIDIA RTX PRO 6000 to train LLaMA-7B model. We measure the memory consumption of three key components: optimizer states, gradients, and model parameters. The ``Others'' category includes remaining memory overhead primarily from activations, caches, and system allocations. Due to out-of-memory (OOM) errors encountered during full-parameter training, we estimate the memory breakdown for this baseline based on the theoretical proportional relationships between optimizer states, gradients, and model parameters (2:1:1).

\textbf{GPU memory consumption.} To evaluate GPU memory cost of different memory-efficient methods, we conduct experiments using a single NVIDIA RTX PRO 6000 to pre-train the LLaMA-7B model on the C4 dataset. We use a micro batch size of 16 with gradient accumulation steps of 32, resulting in an effective total batch size of 512. We measure the memory consumption of three key components: optimizer states, gradients, and model parameters. The ``Others'' category includes remaining memory overhead primarily from activations, caches, and system allocations. For memory-efficient methods including GaLore and GoLore, we set the rank to 128. For LISA and LISA-wor, we set the sampling layers to 2, selecting from the 32 middle layers of LLaMA-7B. The detailed memory breakdown for each method is presented in Table~\ref{tab:memory_breakdown}.

\begin{table}[h]
\centering
\caption{Memory breakdown (GB) for different training methods when pre-training LLaMA-7B on C4 dataset. }
\label{tab:memory_breakdown}
\begin{tabular}{lccccc}
\toprule
\textbf{Method} & \textbf{Model} & \textbf{Gradients} & \textbf{Optimizer} & \textbf{Others} & \textbf{Total}  \\
\midrule
Full params & 12.55 & 12.55 & 25.10 & 14.66 & 64.86  \\
GaLore/GoLore & 12.55 & 12.55 & 1.73 & 4.40 & 31.23  \\
LISA/LISA-wor & 12.55 & 1.24 & 2.48 & 3.29 & 19.56  \\
\bottomrule
\end{tabular}
\end{table}

%%%%%%%%%%%%%%%%%%%%%%%%%%%%%%%%%%%%%%%%%%%%%%%%%%%%%%%%%%%%%%%%%%%%%%%%%%%%%%%
%%%%%%%%%%%%%%%%%%%%%%%%%%%%%%%%%%%%%%%%%%%%%%%%%%%%%%%%%%%%%%%%%%%%%%%%%%%%%%%

\end{document}